\documentclass{article}

\PassOptionsToPackage{numbers,compress}{natbib}

\usepackage[main,final]{neurips_2026}

\usepackage[utf8]{inputenc}
\usepackage[T1]{fontenc}

\usepackage{microtype}
\usepackage{graphicx}
\usepackage{subcaption}
\usepackage{booktabs}
\usepackage{multirow}
\usepackage{hyperref}
\usepackage{url}
\usepackage{xcolor}
\usepackage{float}
\usepackage{comment}
\usepackage{amsmath}
\usepackage{amssymb}
\usepackage{amsfonts}
\usepackage{mathtools}
\usepackage{amsthm}

\usepackage[capitalize,noabbrev]{cleveref}
\usepackage{enumitem}
\setlist[itemize]{noitemsep, topsep=0pt}

\usepackage{etoolbox}
\BeforeBeginEnvironment{proof}{\vspace{-\baselineskip}}
\AfterEndEnvironment{proof}{\vspace{-\baselineskip}}
\theoremstyle{plain}
\newtheorem{theorem}{Theorem}[section]
\newtheorem{proposition}[theorem]{Proposition}

\theoremstyle{definition}

\theoremstyle{remark}
\newtheorem{remark}[theorem]{Remark}

\newcommand{\R}{\mathbb{R}}
\newcommand{\N}{\mathbb{N}}
\newcommand{\K}{\mathbb{K}}
\newcommand{\C}{\mathbb{C}}
\DeclareMathOperator{\softplus}{softplus}
\newcommand{\pp}{\,pp}
\usepackage{tabularx}
\newcommand{\accbar}[3][green!30]{\rlap{\textcolor{#1}{\rule[-0.35ex]{#2}{1.8ex}}}\hspace{0.15em}#3}

\title{LogSig-SSM: Time-Series Modelling with Multi-Scale Log-Signature Compression for State-Space Models}

\author{%
  Felix Oury \\
  Department of Computing \\
  Imperial College London \\
  \texttt{felix.oury19@imperial.ac.uk}
  \And
  Nicolas Calvo Peiro \\
  Department of Bioengineering \\
  Imperial College London \\
  \texttt{nicolas.calvo-peiro19@imperial.ac.uk}
  \And
  Reiko J. Tanaka \\
  Department of Bioengineering \\
  Imperial College London \\
  \texttt{r.tanaka@imperial.ac.uk}
}

\begin{document}

\maketitle

\begin{abstract}
Time-series data are often sampled irregularly at high frequencies and exhibit long-range dependencies, which makes long-horizon modelling difficult. Continuous-time models such as neural controlled differential equations (NCDEs) and neural rough differential equations (NRDEs) can handle irregular sampling, but they scale poorly to long sequences. Selective state-space models (SSMs) such as Mamba scale linearly with sequence length, but they provide limited recurrent mixing across hidden dimensions within a single block. We propose LogSig-SSM (Log-Signature Compression for State-Space Models), which first compresses long multivariate time series into a shorter sequence of tokens using multi-scale windowed log-signatures, and then processes these tokens with a selective SSM backbone. LogSig-SSM is scalable and robust to irregular sampling, combining log-signature tokens that capture higher-order cross-channel interactions with a selective SSM that models long-range dependencies. The model also admits a continuous-time interpretation as an NCDE/NRDE-style system driven by a log-signature-based input, in which selectivity induces an input-dependent rescaling of the latent dynamics. Across four benchmarks, namely long-sequence classification on UEA, high-frequency physiological regression on PPG-DaLiA, multivariate weather forecasting, and irregularly sampled clinical prediction on PhysioNet Sepsis, LogSig-SSM outperforms or matches strong SSM and continuous-time baselines
while training up to $30\times$ faster and using up to $37\times$ less GPU memory than Mamba on the longest sequences.\footnote{Code: \url{https://github.com/felix-oury/LogSigSSM}. }

\end{abstract}

\section{Introduction}
\label{sec:introduction}

Time-series modelling has numerous real-world applications in domains such as healthcare \cite{dunn_wearable_2021} and finance \cite{hautsch_modelling_2004}. These applications typically involve signals, such as ECG or transaction prices, that are irregularly sampled at high frequencies and contain information spanning long time scales. Processing these long sequences creates memory and computational bottlenecks which motivate the development of models that scale to long sequences while remaining reliable under irregular sampling rates. Two widely-used approaches are continuous-time models and State-Space-Models (SSMs).

Continuous-time models such as Neural Ordinary Differential Equations (NODEs) \cite{chen2019neuralordinarydifferentialequations} and Neural Controlled Differential Equations (NCDEs) \cite{kidger_neural_2020} treat observed time-series data as discrete-time samples from an underlying continuous-time process. NCDEs operate on continuous interpolations of the data and naturally accommodate irregularly-sampled observations.
However, training and inference can become costly on longer sequences because numerical solvers often require more integration steps. To address this computational bottleneck, Neural Rough Differential Equations (NRDEs) \cite{morrill_neural_2021} and Log-NCDEs \cite{walker_log_2024} 
employ the log-ODE method \cite{chevyrev_primer_2025,friz_multidimensional_2010} which summarises information between time intervals with truncated log-signatures, enabling the solver to take fewer, larger steps.
Signature-based summarisation has also been adapted to improve computational efficiency and handle irregular sampling in discrete models such as RNNs \cite{Logsigrnn} and transformers \cite{moreno-pino_rough_2024}.

In parallel, structured SSMs \cite{boyer_learning_2025,dao_transformers_2024,gu_mamba_2024,gu_efficiently_2022,rusch_oscillatory_2025,smith_simplified_2023} have emerged as an efficient alternative to attention-based models. Unlike transformers, whose computational cost scales quadratically \cite{DBLP:journals/corr/VaswaniSPUJGKP17}, modern SSMs achieve linear-time processing, constant-memory inference, and efficient GPU-parallel training. SSMs have also shown competitive performance across various sequence modelling tasks~\cite{lieber2024jambahybridtransformermambalanguage,schiff2024caduceusbidirectionalequivariantlongrange, Shams_2024,zuo2024falconmambacompetitiveattentionfree}. Recent selective SSMs, such as Mamba \cite{gu_mamba_2024}, introduced input-dependent gating enabling the model to filter context content rather than uniformly compressing it.
Additionally, novel theoretical work \cite{cirone_theoretical_2025} has established that SSMs admit a continuous-time interpretation as linear controlled differential equations (CDEs), connecting them to the continuous-time modelling framework.

However, recent empirical studies highlight two key limitations of Mamba. First, performance can degrade on very long sequences, where Mamba-based models may underperform attention-based baselines \cite{waleffe2024empiricalstudymambabasedlanguage}. Existing approaches mitigate this limitation through token pruning or decimation \cite{ye2025longmambaenhancingmambaslong,benkish2025decimambaexploringlengthextrapolation}. Second, Mamba's diagonal state transition admits no within-block mixing across hidden dimensions. As \citet[Cor 4.4]{cirone_theoretical_2025} show, the closure of a single Mamba block decomposes as a sum of per-channel functionals. This means genuinely multivariate path functionals, such as Lévy areas, which depend on the joint behaviour of two channels, lie outside this closure. Two mitigations exist for this limitation, but come with costs. Stacking blocks with linear inter-block mixing eventually recovers full expressivity (Prop. 4.5 in \cite{cirone_theoretical_2025}), but only as both depth and hidden-state size grow without bound, i.e. cross-channel structure is paid for in model size. Replacing the diagonal state transition with richer structure \cite{walker2025structuredlinearcdesmaximally,terzic2025structuredsparsetransitionmatrices} avoids the depth cost, but at non-trivial efficiency overhead (PD-SSM reports approximately $6\times$ slower runtime than Mamba \cite{terzic2025structuredsparsetransitionmatrices}). These limitations motivate a complementary strategy: compressing the effective sequence length while injecting explicit higher-order cross-channel information at the input, leaving the efficient diagonal scan intact. 

Log-signature features \cite{lyons_differential_2007} provide exactly such input-level functionals. They encode iterated integrals of the path, including Lévy areas, while inheriting invariance to monotone time reparameterisations and thus offering robustness to irregular sampling. Building on this idea, we introduce LogSig-SSM, a two-stage architecture that achieves both accuracy and efficiency on long sequences. First, a long multivariate sequence is mapped to a shorter token sequence by computing truncated log-signatures over strided windows at multiple scales, yielding compact descriptors that explicitly capture cross-channel dependencies. Second, these tokens are processed by a diagonal selective SSM, which efficiently models long-range dependencies over the reduced sequence length.

Our contributions are:
\begin{itemize}
    \item We introduce LogSig-SSM, a model that integrates log-signature tokenisation with selective SSMs. The tokeniser compresses long, irregularly-sampled multivariate sequences into compact summaries by computing truncated log-signatures over multi-scale windows.
    \item We provide a continuous-time interpretation of LogSig-SSM by extending the linear-CDE view of SSMs \cite{cirone_theoretical_2025}, interpreting the model as an NCDE/NRDE-style system driven by log-signature-derived controls with input-dependent time rescaling of the latent dynamics.
    \item We establish a depth--truncation tradeoff. A single LogSig-SSM block followed by an MLP can approximate any continuous functional of the windowed log-signature up to depth $m$, while the same class of functionals lies outside the closure of a single diagonal Mamba block (even with a pointwise MLP, which cannot recover features absent from the block state) and requires chained blocks to recover.
    \item We evaluate LogSig-SSM on four benchmarks: long-sequence time-series classification (UEA), high-frequency physiological signal regression (PPG-DaLiA), multivariate weather forecasting, and irregularly sampled early sepsis prediction (PhysioNet Sepsis). LogSig-SSM matches or exceeds strong SSM and continuous-time baselines 
    and reduces training time by $30\times$ and GPU memory by $37\times$ relative to standard Mamba on the longest sequences. 
\end{itemize}
\section{Theoretical Background}
\label{sec:background}

We review the theoretical foundations of LogSig-SSM: (i) structured SSMs, which serve as the sequential processing backbone, (ii) their continuous-time interpretation as linear CDEs, and (iii) log-signatures, which provide compact summaries of multivariate time-series over time intervals.

\subsection{Structured State-Space Models}
\label{sec:ssm_background}

Modern SSM sequence layers can be viewed as discretisations of a continuous-time linear state-space system.
Let $X:[0,1]\to\R^{d}$ be an input signal observed at time points $0\le t_0<\cdots<t_L\le 1$.
We write $x_\ell := X(t_\ell)$ for the resulting discrete sequence and $x_{i,\ell}:=X_i(t_\ell)$ for its $i$th channel. A channel-wise linear SSM processes each channel $X_i$ with its own latent state $H_i:[0,1]\to\K^{N}$ and produces a scalar output $y_i:[0,1]\to\R$:
\begin{equation}
\begin{gathered}
\frac{dH_i(t)}{dt} = A_i H_i(t) + B_i\,X_i(t),
\quad
y_i(t) = C_i H_i(t),
\end{gathered}
\label{eq:ssm_ct_channelwise}
\end{equation}
where $A_i\in\K^{N\times N}$ is structured (often diagonal), $B_i\in\K^{N}$, $C_i\in\K^{1\times N}$, and $\K\in\{\R,\C\}$. To apply SSMs to discrete sequence data, we discretise the continuous-time dynamics using a step size $\Delta_i>0$.
Under a zero-order hold (ZOH) discretisation with step $\Delta_i$:
\begin{equation}
H_{i,\ell} = \overline A_i H_{i,\ell-1} + \overline B_i x_{i,\ell},\quad
\overline A_i = e^{\Delta_i A_i},\quad
\overline B_i = (\Delta_i A_i)^{-1}(e^{\Delta_i A_i}-I)\Delta_i B_i.
\label{eq:ssm_dt_channelwise}
\end{equation}

A key design choice in modern SSMs is how the continuous-dynamics are discretised. Linear Time-Invariant (LTI) SSMs, such as S4 \cite{gu_efficiently_2022} and S5 \cite{smith_simplified_2023}, use a fixed step size $\Delta_i$ learned at training time, giving time-invariant operators $(\overline A_i, \overline B_i)$. Selective SSMs, such as Mamba \cite{gu_mamba_2024}, instead make the step size input-dependent. This induces input-dependent selectivity, meaning the model can adapt its update rate across the sequence, focusing on important information while down-weighting irrelevant inputs. In addition, the input and output projections $s_B:\R^{d}\to\R^{N}$ and $s_C:\R^{d}\to\R^{1\times N}$ are learned functions of the current sample $x_\ell$.
In Mamba, the channel-wise internal step is defined as
\begin{equation}
\Delta_{i,\ell} = \mathrm{softplus}(\alpha_i^\top x_\ell + \beta_i),\qquad \alpha_i\in\R^d,\ \beta_i\in\R,
\label{eq:delta_map_bg}
\end{equation}
yielding input-dependent operators,
\begin{equation}
\overline{A}_{i,\ell}=\exp(\Delta_{i,\ell}A_i),
\quad
\overline{B}_{i,\ell} = (\Delta_{i,\ell}A_i)^{-1}\!\bigl(\exp(\Delta_{i,\ell}A_i)-I\bigr)\,\Delta_{i,\ell}\,s_B(x_\ell),
\quad
\overline{C}_{i,\ell}=s_C(x_\ell).
\label{eq:Bbar_zoh_bg}
\end{equation}
For small $\Delta_{i,\ell}$, $\overline{B}_{i,\ell}\approx\Delta_{i,\ell}s_B(x_\ell)$, so selectivity acts as input-dependent scaling of the state update.

\subsection{Unifying View: Linear CDEs and SSMs}
\label{sec:ssm_cde_bridge}

SSM layers admit a continuous-time interpretation as linear CDEs ~\cite{cirone_theoretical_2025}, a connection we extend in Section~\ref{sec:sig_ssm_continuous} to derive LogSig-SSM's continuous-time form. 
Let $H_t \in \R^{N}$ denote the hidden state and let $\omega: [0,1] \to \R^{d_\omega}$ and $\xi: [0,1] \to \R^{d_\xi}$ be control paths. We consider the linear CDE
\begin{equation}
dH_t \;=\; \sum_{i=1}^{d_\omega} A_i H_t \, d\omega_t^{i} \;+\; W_B\, d\xi_t,
\qquad
H_0 \in \R^{N},
\label{eq:ssm_linear_cde}
\end{equation}
for learnable matrices $A_i \in \R^{N \times N}$ and $W_B \in \R^{N \times d_\xi}$\footnote{
To avoid overloading notation, we use $W_B$ for the time-invariant injection matrix in the linear CDE view (denoted $B$ in \cite{cirone_theoretical_2025}).
We reserve $s_B(\cdot)$ for Mamba's input-dependent projection in selective scans, written as $B$ in \cite{gu_mamba_2024}.
}. Selective SSMs correspond to input-dependent driving paths $(\omega^X, \xi^X)$. Viewed in this continuous-time framework, selectivity can be interpreted as an input-dependent adaptive time rescaling: the model dynamically adjusts how quickly its latent state evolves based on the characteristics of the input \cite{cirone_theoretical_2025}.

\subsection{Path Signatures and Log-Signatures}
\label{sec:sig_background}

The signature transform represents a path through its hierarchy of iterated integrals \cite{lyons_differential_2007, chevyrev_primer_2025}.  Formally, let $X \in C^{1}([0, 1];\R^{d})$ be a path and fix $0 \le a < b \le 1$. The path signature of $X$ over $[a,b]$ is the (infinite) collection of iterated integrals:
\begin{equation}
\begin{gathered}
\mathrm{Sig}(X)_{a,b} := \Big(1, \mathbf{X}^{1}_{a,b}, \mathbf{X}^{2}_{a,b}, \ldots\Big), \quad
\mathbf{X}^{m}_{a,b} := \int_{a < u_1 < \cdots < u_m < b} dX_{u_1} \otimes \cdots \otimes dX_{u_m},
\end{gathered}
\label{eq:sig_def_bg}
\end{equation}
where $\mathbf{X}^{m}_{a,b} \in (\R^d)^{\otimes m}$.
In practice, we truncate the signature at depth $M$ and write $\mathrm{Sig}_M(X)_{a,b}:=\bigl(1,\mathbf{X}^1_{a,b},\ldots,\mathbf{X}^M_{a,b}\bigr)$. The signature is invariant under monotone time reparameterisations \cite{lyons_differential_2007}. If $\gamma: [0, 1] \to [a,b]$ is a continuous, non-decreasing, surjective reparameterisation, then
\begin{equation}
\mathrm{Sig}(X)_{a,b} = \mathrm{Sig}(X\circ\gamma)_{0,1}.
\label{eq:sig_time_reparam_invariance}
\end{equation}
This ensures the signature depends on the geometric path traced over $[a,b]$ rather than the sampling density, providing natural robustness to irregular sampling once a path interpolation is fixed.

The log-signature is the tensor-algebra logarithm $\mathrm{LogSig}(X)_{a,b} := \log(\mathrm{Sig}(X)_{a,b})$, which lies in the free Lie algebra and, after choosing a Lie basis, provides a compact coordinate representation of the same information. Its truncation inherits the same reparameterisation invariance:
\begin{equation}
\mathrm{LogSig}_{M}(X)_{a,b} = \mathrm{LogSig}_{M}(X \circ \gamma)_{0,1}.
\label{eq:logsig_time_reparam_invariance}
\end{equation}
In LogSig-SSM (Section~\ref{sec:sig_ssm_discrete}), we compute truncated log-signatures on overlapping windows at multiple scales to form a compressed token stream.
\section{LogSig-SSM Methodology}
\label{sec:sig_ssm_framework}

\begin{figure*}
  \centering
  \includegraphics[width=1\textwidth]{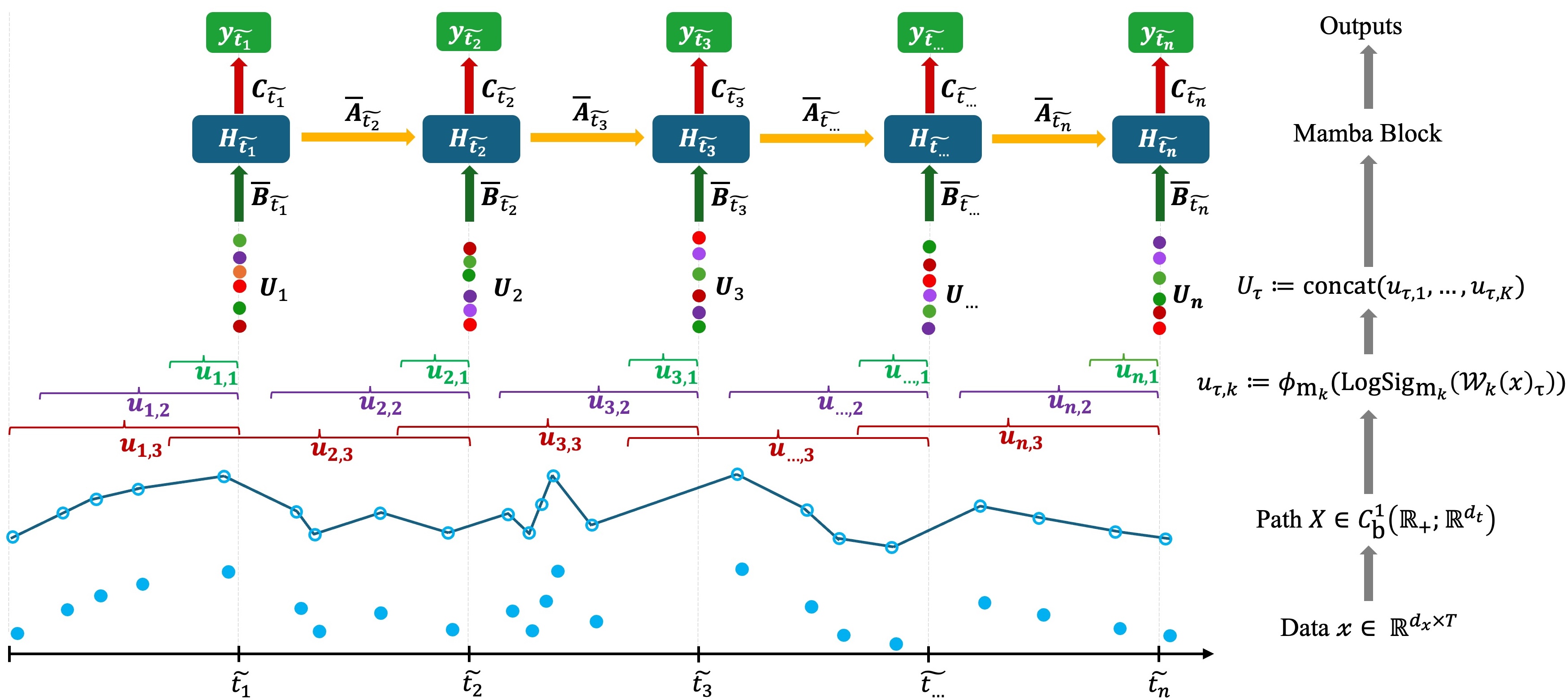}
  \caption{Discrete-time LogSig-SSM pipeline. A multivariate time series $x\in\R^{d_x\times T}$ is summarised at strided anchors $\tilde t_\tau:=t_{\tau s}$. For each scale $k$, a windowed segment $\mathcal{W}_k(x)_\tau$ is extracted, and its truncated log-signature coordinate vector, $u_{\tau,k} := \phi_{m_k}\!\Big(\mathrm{LogSig}_{m_k}\!\big(\mathcal{W}_k(x)_\tau\big)\Big)\in \R^{D_k},$ is computed. These are concatenated across scales to form the token $U_\tau := \mathrm{concat}\big(u_{\tau,1},\dots,u_{\tau,K}\big)\in\R^{D_{\mathrm{tot}}}$. Finally, $U_\tau$ is projected to the scan width $d$ via a learned linear embedding to form $\widetilde U_\tau = W_{\mathrm{in}}U_\tau$, which serves as the input to the Mamba block.}
  \label{fig:sigssm_discrete_overview}
\end{figure*}

We introduce LogSig-SSM, a selective SSM driven by log-signature tokenisation. LogSig-SSM (i) compresses a multivariate input sequence into a sequence of interval descriptors using multi-scale windowed log-signatures, and
(ii) processes the resulting token stream with a selective SSM backbone for efficient long-range propagation. Figure~\ref{fig:sigssm_discrete_overview} summarises the discrete-time implementation.

\subsection{Discrete-Time Implementation of LogSig-SSM}
\label{sec:sig_ssm_discrete}

Let $x\in\R^{d_x\times T}$ be a multivariate time-series with samples $\{x_j\}_{j=0}^{T-1}$ observed at times
$0=t_0<\cdots<t_{T-1}$ (not necessarily uniformly spaced).
Fix a stride $s\in\N$ and define anchor indices $\tau\in\{1,\dots,T'\}$ with
$T':=\lfloor \frac{T-1}{s}\rfloor$ and anchor times $\tilde t_\tau:=t_{\tau s}$.
Fix $K\in\N$ window scales, specified by hyperparameters $\{L_k\}_{k=1}^K$ (window lengths) and
$\{m_k\}_{k=1}^K$ (log-signature truncation depths), where the index $k$ enumerates scales.
For each scale $k$ and anchor $\tau$, let $W_{k,\tau}$ denote the windowed samples,
$
W_{k,\tau} := \{(t_j,x_j):\ \max(0,\tilde t_\tau-L_k)\le t_j \le \tilde t_\tau\}.
$
Let $\mathcal{W}_k(x)_\tau:[\max(0,\tilde t_\tau-L_k),\tilde t_\tau]\to\R^{d_x}$ be the bounded-variation
path obtained by piecewise-linear interpolation of $W_{k,\tau}$. Define its truncated log-signature by
$
\mathrm{LogSig}_{m_k}\!\big(\mathcal{W}_k(x)_\tau\big)\in \mathfrak{g}^{(\le m_k)}(\R^{d_x}),
$
where $\mathfrak{g}^{(\le m)}(\R^{d_x})$ is the step-$m$ truncated free Lie algebra over $\R^{d_x}$.
We tokenise with the log-signature rather than the signature. The two are related by a polynomial bijection and carry the same information at equal depth (Section~\ref{sec:sig_background}). The log-signature is the redundancy-free coordinate system for that content and the driver of the log-ODE method underlying the NRDE and Log-NCDE models we build on.
In implementation the log-signature is represented by its coordinate vector in a fixed basis: let $\phi_{m_k}:\mathfrak{g}^{(\le m_k)}(\R^{d_x})\to\R^{D_k}$ be the corresponding coordinate map, and define
\[
u_{\tau,k} := \phi_{m_k}\!\Big(\mathrm{LogSig}_{m_k}\!\big(\mathcal{W}_k(x)_\tau\big)\Big)\in\R^{D_k}.
\]
Stacking across $K$ scales yields the LogSig-SSM token
\begin{equation}
\begin{gathered}
U_\tau := \mathrm{concat}\big(u_{\tau,1},\dots,u_{\tau,K}\big)\in\R^{D_{\mathrm{tot}}},
\qquad
D_{\mathrm{tot}}:=\sum_{k=1}^K D_k .
\label{eq:sigssm_token}
\end{gathered}
\end{equation}

The model dimension (scan width) $d$ is a design choice and may differ from $D_{\mathrm{tot}}$.
Accordingly, the scan input is obtained by applying a learned linear embedding to each token,
\begin{equation}
\widetilde U_\tau := W_{\mathrm{in}}\,U_\tau \in \R^{d},
\qquad
W_{\mathrm{in}}\in\R^{d\times D_{\mathrm{tot}}}.
\label{eq:sigssm_token_embed}
\end{equation}

The resulting sequence $\widetilde U_{1:T'}$ is viewed as $d$ channels, indexed by $i\in\{1,\dots,d\}$.
The token stream is then processed by a stack of Mamba blocks \cite{gu_mamba_2024}, each governed by the selective scan from Section~\ref{sec:ssm_background}.
The discretisation step $\Delta_{i,\tau}\in\R_{+}$, the input projection $s_B(\widetilde U_\tau)\in\R^{N}$, and the output projection $s_C(\widetilde U_\tau)\in\R^{1\times N}$ are all functions of the current input $\widetilde U_\tau\in\R^{d}$, where $N$ denotes the state size.

Referencing the ZOH recurrence derived in \eqref{eq:ssm_dt_channelwise}, the per-channel state $H_{i,\tau}$ evolves by
\begin{equation}
\begin{gathered}
H_{i,\tau} = \overline A_{i,\tau}\,H_{i,\tau-1} + \overline B_{i,\tau}\,\widetilde U_{i,\tau}, \qquad
y_{i,\tau} = s_C(\widetilde U_\tau)\, H_{i,\tau},
\end{gathered}
\label{eq:sigssm_ssm_form}
\end{equation}
where $\overline A_{i,\tau} = \exp(\Delta_{i,\tau}A_i)$.
Using the small-step approximation from Section~\ref{sec:ssm_background}, the driving term simplifies to:
$
\overline B_{i,\tau}\,\widetilde U_{i,\tau} \;\approx\; \Delta_{i,\tau}\,s_B(\widetilde U_\tau)\,\widetilde U_{i,\tau}.
$

\subsection{Continuous-time Interpretation of LogSig-SSM}
\subsubsection{LogSig-SSM as a Linear CDE}
\label{sec:sig_ssm_continuous}
To obtain a continuous-time interpretation, we adopt the small-step approximation (Section~\ref{sec:ssm_background}) standard in selective SSM analysis \cite{cirone_theoretical_2025}.
Consider the sequence of tokens $\{U_\tau\}_{\tau=1}^{T'}$ from \eqref{eq:sigssm_token}. We define a piecewise-constant control path $U:[0,1]\to\R^{D_{\mathrm{tot}}}$ by choosing a partition
$0=r_0<\cdots<r_{T'}=1$ and setting:
\begin{equation}
U(t) = U_\tau, \qquad t \in [r_{\tau-1}, r_\tau) \quad \text{for } \tau=1,\dots,T'.
\label{eq:sigssm_U_path}
\end{equation}
The effective driving signal is then the projection $\widetilde U(t) := W_{\mathrm{in}}\,U(t) \in \R^{d}$.
As shown in Appendix~\ref{app:linear_cde_derivation},
LogSig-SSM admits the Linear CDE representation
\begin{equation}
\begin{gathered}
dH^{\widetilde U}_{i;t}=A_i H^{\widetilde U}_{i;t}\,d\omega^{\widetilde U}_{i;t} +W_B\,d\xi^{\widetilde U}_{i;t}, \qquad
H^{\widetilde U}_{i;0}=h_{i,0},
\qquad i=1,\dots,d,
\end{gathered}
\label{eq:sigssm_linear_cde}
\end{equation}
where $W_B\in\R^{N\times d}$ is the injection matrix, consistent with the discrete-time projection $s_B(u)=W_B\,u$.
The system is driven by \begin{equation}
\begin{gathered}   
\omega^{\widetilde U}_{i;t}
:=
\int_0^t \sigma(\widetilde\alpha_i^\top \widetilde U_s+\widetilde\beta_i)\,ds
\in\R, \qquad
\xi^{\widetilde U}_{i;t}
:=
\int_0^t \widetilde U_s\,\widetilde U_s^i\,\sigma(\widetilde\alpha_i^\top \widetilde U_s+\widetilde\beta_i)\,ds
\in\R^{d},
\end{gathered}
\label{eq:sigssm_drivers_mamba_style}
\end{equation}
for parameters $\widetilde\alpha_i\in\R^{d}$ and $\widetilde\beta_i\in\R$.
Since $\widetilde U(t)=W_{\mathrm{in}}\,U(t)$ and $U$ is constructed from multi-scale windowed log-signatures of the original path,
the drivers remain log-signature-induced functionals of the original trajectory.

\subsubsection{Relation to Neural RDEs: Interval-Driven Dynamics}
Under the ZOH extension \eqref{eq:sigssm_U_path}, the drivers $(\omega^{\widetilde U}_i,\xi^{\widetilde U}_i)$ are affine on each interval with constant derivatives since $\widetilde U_s=\widetilde U_\tau$ for $s\in[r_{\tau-1},r_\tau)$. Summarising each interval by the driver increments $\Lambda^{\widetilde U}_{i,\tau}:=(\Delta\omega^{\widetilde U}_{i,\tau},\,\Delta\xi^{\widetilde U}_{i,\tau})^\top$, equation \eqref{eq:sigssm_linear_cde} reduces to the NRDE-style integral form for $t\in(r_{\tau-1},r_\tau)$:
\begin{equation}
H^{\widetilde U}_{i;t}
=
H^{\widetilde U}_{i;r_{\tau-1}}
+
\int_{r_{\tau-1}}^{t}
\big[\,A_i H^{\widetilde U}_{i;s},\ W_B\,\big]\;
\frac{1}{\Delta r_\tau}\Lambda^{\widetilde U}_{i,\tau}\;ds.
\label{eq:sigssm_nrde_interval_integral_33}
\end{equation}
In particular, $\Delta\omega^{\widetilde U}_{i,\tau}=\sigma(\widetilde\alpha_i^\top\widetilde U_\tau+\widetilde\beta_i)\Delta r_\tau$ acts as a content-dependent internal-time allocation per interval. LogSig-SSM's construction is inspired by NRDEs \cite{morrill_neural_2021} but introduces two key differences. First, NRDEs use the log-signature of the raw path over a single window per interval. LogSig-SSM instead uses a multi-scale token $U_\tau$ concatenating log-signatures at $K$ different window lengths, providing descriptors at complementary resolutions within a single token. Second, the drivers $\Lambda^{\widetilde U}_{i,\tau}$ are content-dependent, unlike the fixed interval parametrisation in NRDEs. Therefore, LogSig-SSM learns to allocate internal time selectively across the compressed sequence. The derivation of \eqref{eq:sigssm_nrde_interval_integral_33} from \eqref{eq:sigssm_linear_cde} is provided in Appendix~\ref{app:nrde_interval_derivation}.
\section{Theoretical Analysis}
\label{sec:theoretical_analysis}
We highlight three theoretical attributes of LogSig-SSM. First, we establish a depth-truncation tradeoff by showing that injecting cross-channel path geometry directly into the tokens allows LogSig-SSM to expose non-commutative multivariate features in a single diagonal block that would otherwise require deep, chained Mamba blocks \cite{cirone_theoretical_2025}. Second, tokenisation reduces the effective sequence length, cutting selective scan complexity by a factor of $T/T'$ relative to Mamba. Third, LogSig-SSM inherits within-window time-reparameterisation invariance from the log-signature transform, providing robustness to irregular sampling.
\subsection{Single-layer expressivity via a log-signature driver lift}
\label{sec:th_expressivity}
A single diagonal Mamba block cannot represent non-commutative
cross-channel iterated integrals of its input path; recovering them
requires chained Mamba blocks (Cor.~4.4 and Sec.~C.2.1 of
\citet{cirone_theoretical_2025}). LogSig-SSM bypasses this depth cost
by lifting the input before the scan. By construction, the token
$U_\tau$ contains every log-signature coordinate of the windowed path
$\mathcal{W}_k(X)_\tau$ up to depth $m_k$. The following proposition
formalises the resulting depth--truncation tradeoff.

\begin{proposition}[Depth--truncation tradeoff; Proof in App.~\ref{app:expressivity_proof}]
\label{prop:th_mixing_tradeoff}
Fix a compact family
$\mathcal{X}\subset C^{1}_{b}([0,T_{\mathrm{end}}];\mathbb{R}^{d_x})$
with $d_x\geq 2$, a window scale $(L_k, m_k)$, an anchor $\tau$,
and $2\leq m\leq m_k$. Let $F:\mathcal{X}\to\mathbb{R}$ be a
continuous functional that factors as
$F(X) = f\!\left(\mathrm{LogSig}_m(\mathcal{W}_k(X)_\tau)\right)$
for some continuous $f$. Then:
\textbf{(i)} for every $\varepsilon>0$, a single LogSig-SSM block
followed by a post-scan MLP approximates $F$ uniformly on
$\mathcal{X}$ to error~$\varepsilon$;
\textbf{(ii)} there exist such $F$ with $F\notin
\mathcal{F}^{(1)}_{\mathrm{Mamba}}[\mathcal{X}]$, where
$\mathcal{F}^{(1)}_{\mathrm{Mamba}}[\mathcal{X}]$ denotes the
uniform closure on $\mathcal{X}$ of functionals computable by a
single diagonal Mamba block on the raw input path
(Cor.~4.4 of \cite{cirone_theoretical_2025});
\textbf{(iii)} any raw-input Mamba approximation of such $F$
requires chained diagonal blocks; recovering depth-$m$
cross-channel features in particular requires chains of depth
growing with $m$ in the construction of
\citet{cirone_theoretical_2025} (Sec.~C.2.1, Prop.~4.5),
whereas LogSig-SSM realises $F$ in a single block whenever
$m_k \geq m$.
\end{proposition}

\begin{remark}[Multi-scale and asymptotic universality]
\label{rem:tradeoff}
Proposition~\ref{prop:th_mixing_tradeoff} reflects a depth--truncation
tradeoff, not an expressivity gap. The benefit is that LogSig-SSM
moves the cost of cross-channel geometry from the online selective
scan to an offline, parallel tokeniser. Across scales, any
depth-$\leq m_k$ coordinate at scale $k$ enters $U_\tau$ linearly;
thus, a single LogSig-SSM block captures features that would
otherwise require a Mamba chain whose depth grows with $\max_k m_k$.
By signature universality \citep{lyons_differential_2007}, in the
limit $m_k\to\infty$ a single LogSig-SSM block approximates any
continuous functional of windowed paths, consistent with the
asymptotic universality of chained diagonal Mamba on raw $X$
(Prop.~4.5 of \citet{cirone_theoretical_2025}).
\end{remark}
\subsection{Computational Efficiency and Scan Complexity}
\label{sec:th_compression}

The selective scan scales linearly with sequence length and constitutes the main time-sequential component of a Mamba layer \cite{gu_mamba_2024}. Recent studies indicate that Mamba-style models can degrade at extreme context lengths, which has been linked to length-dependent state decay and a reduced effective memory horizon \cite{benkish2025decimambaexploringlengthextrapolation,waleffe2024empiricalstudymambabasedlanguage,ye2025longmambaenhancingmambaslong}. Consequently, reducing the effective sequence length is a critical design component for selective SSMs. 
LogSig-SSM performs this compression explicitly by replacing the length-$T$ input sequence with a token sequence of length $T'\approx T/s$, where $s\in\N$ is the stride. Since the selective scan cost scales as $O(B_{\mathrm{sz}}\cdot T\cdot N\cdot D)$, compression from $T$ to $T'$ reduces scan FLOPs by a factor of $T/T'\approx s$.
Additionally, log-signature tokenisation is a one-time offline pre-processing step: precomputed features are not in the training-time backpropagation graph, and computation is highly optimised in libraries such as \texttt{iisignature}~\cite{reizenstein2018iisignaturelibraryefficientcalculation} and \texttt{Signatory}~\cite{kidger2021signatorydifferentiablecomputationssignature}.

\subsection{Robustness to Irregular Sampling}
\label{sec:th_sampling}

Since log-signatures are invariant under monotone time reparameterisations (Section~\ref{sec:sig_background}), tokens computed from windowed paths encode path geometry rather than sampling density. Under a fixed interpolation, two observation sequences whose windowed paths coincide up to a monotone reparameterisation yield identical tokens and therefore identical model outputs.

\begin{proposition}[Within-window invariance to monotone time reparameterisation]
\label{prop:th_sampling_invariance}
Let $\widehat X,\widehat X':[0,T_{\mathrm{end}}]\to\R^{d_x}$ be two bounded-variation interpolants and $I_{\tau,k} := [\tilde t_\tau-L_k,\tilde t_\tau]$ the time interval for window $(\tau,k)$. For every scale $k$ and anchor $\tau$, assume there exists a continuous, non-decreasing, surjective map $\gamma_{\tau,k}: I_{\tau,k} \to I_{\tau,k}$ such that
$
\mathcal{W}_k(\widehat X')_\tau(t) = \mathcal{W}_k(\widehat X)_\tau(\gamma_{\tau,k}(t)), \quad  \forall t \in I_{\tau,k}.
$
Then $U_\tau(\widehat X)=U_\tau(\widehat X')$ for all $\tau$, and $\mathrm{LogSigSSM}(\widehat X)=\mathrm{LogSigSSM}(\widehat X')$.
\end{proposition}
\begin{proof}
The reparameterisation invariance of the signature (cf.\ \eqref{eq:sig_time_reparam_invariance}) gives $\mathrm{Sig}(\mathcal{W}_k(\widehat X)_\tau)=\mathrm{Sig}(\mathcal{W}_k(\widehat X')_\tau)$. Taking the truncated log-signature and applying the fixed coordinate map $\phi_{m_k}$ yields $u_{\tau,k}(\widehat X)=u_{\tau,k}(\widehat X')$. Concatenating across scales gives $U_\tau(\widehat X)=U_\tau(\widehat X')$, and the claim follows from the determinism of the selective scan.
\end{proof}

The theoretical guarantee is local to each window. However, Section~\ref{sec:physionet_results} provides empirical evidence that this local invariance is sufficient in practice: LogSig-SSM is highly competitive with specialised irregular-time methods on PhysioNet Sepsis, a dataset where only $10.3\%$ of values are observed.
\section{Experimental Evaluation}
\label{sec:experiments}
 
We empirically test the performance of LogSig-SSM across four settings: long-range multivariate classification, high-frequency physiological regression, multivariate weather forecasting, and irregularly sampled clinical prediction.
Baselines span SSM/recurrent models, continuous-time models, and general time-series methods.
We follow the Log-NCDE protocol~\cite{walker_log_2024} for UEA and PPG-DaLiA, the LinOSS protocol~\cite{rusch_oscillatory_2025} for weather forecasting, and the Neural CDE protocol~\cite{kidger_neural_2020} for PhysioNet Sepsis.
All experiments were done on a NVIDIA RTX~4090 GPU.
Full runtime, memory, and hyperparameter details are in Appendix~\ref{app:memory_runtime} and~\ref{app:hyperparams}.

\subsection{UEA time-series classification}
\label{sec:uea_results}
 
Following the Log-NCDE protocol~\cite{walker_log_2024}, we evaluate on the six longest-sequence datasets of the UEA Archive~\cite{bagnall_uea_2018}, split $70{:}15{:}15$ over five random seeds. LogSig-SSM achieves the best average accuracy across the sixteen models ($71.7\%$) and the best average rank ($2.7$) (Table~\ref{tab:uea_main}).
 
LogSig-SSM demonstrates a distinct advantage on signal types where standard diagonal SSMs severely struggle, most notably on the \textsc{EthanolConcentration} task (4 classes; chance $25\%$). While diagonal SSMs fail to surpass $30\%$ accuracy and the best non-SSM baseline reaches $40.5\%$, LogSig-SSM reaches the highest accuracy of $51.4\%$. This signal is an absorption curve whose classes differ by smooth shifts in peak shape over extended regions rather than by high-frequency variation (Figure~\ref{fig:uea_signals}), a global trajectory geometry that multi-scale log-signature tokens capture directly and that diagonal SSM backbones cannot resolve on their own.
 
Beyond accuracy, LogSig-SSM delivers large efficiency gains, summarised for \textsc{EigenWorms} in Table~\ref{tab:worms_efficiency}.
On \textsc{EigenWorms} ($17{,}984$ time points), LogSig-SSM trains in 4\,s per 1000 steps versus 122\,s for Mamba ($30\times$ faster), while using only 362\,MB GPU memory versus 13{,}486\,MB for Mamba ($37\times$ less).
Log-signature tokenisation is a one-time offline step taking $1.6$\,s for the full dataset, and $6$\,s for all six UEA datasets together, about $1\%$ of a single training run per dataset (Table~\ref{tab:uea_sigtime}). These gains come at no accuracy cost: LogSig-SSM leads on \textsc{EigenWorms} while requiring $10\times$ less memory than LinOSS-IM.
\begin{table*}[h]
\centering
\caption{UEA long-sequence classification accuracy (\%\,$\uparrow$), mean\,$\pm$\,std over five random splits, and average rank over the six datasets ($\downarrow$, competition ranking, missing entries ranked last).
Bold\,=\,best, underline\,=\,second best, italic\,=\,third best per column (ties share a mark).
$^\dagger$Transformer OOM on \textsc{EigenWorms}; TSPulse not evaluated on \textsc{EigenWorms}; $^\ast$Model evaluated on 1 seed.}
\label{tab:uea_main}
\scriptsize
\setlength{\tabcolsep}{3.2pt}
\renewcommand{\arraystretch}{0.86}
\begin{tabular}{lcccccccc}
\toprule
Model & EigenWorms & SCP1 & SCP2 & Ethanol & Heartbeat & Motor & Avg & Rank \\
\midrule
Seq.\ length & 17,984 & 896 & 1,152 & 1,751 & 405 & 3,000 & & \\
\# classes   & 5      & 2   & 2     & 4     & 2   & 2     & & \\
\midrule
\multicolumn{9}{l}{\textit{Continuous-time models}} \\
NRDE \cite{morrill_neural_2021}        & $83.9{\scriptstyle\pm7.3}$  & $80.9{\scriptstyle\pm2.5}$ & $53.7{\scriptstyle\pm6.9}$ & $25.3{\scriptstyle\pm1.8}$ & $72.9{\scriptstyle\pm4.8}$ & $47.0{\scriptstyle\pm5.7}$ & $60.6$ & $11.5$ \\
NCDE \cite{kidger_neural_2020}       & $75.0{\scriptstyle\pm3.9}$  & $79.8{\scriptstyle\pm5.6}$ & $53.0{\scriptstyle\pm2.8}$ & $29.9{\scriptstyle\pm6.5}$ & $73.9{\scriptstyle\pm2.6}$ & $49.5{\scriptstyle\pm2.8}$ & $60.2$ & $11.3$ \\
Log-NCDE \cite{walker_log_2024}   & $85.6{\scriptstyle\pm5.1}$  & $83.1{\scriptstyle\pm2.8}$ & $53.7{\scriptstyle\pm4.1}$ & $34.4{\scriptstyle\pm6.4}$ & $75.2{\scriptstyle\pm4.6}$ & $53.7{\scriptstyle\pm5.3}$ & $64.3$ & $7.7$ \\
\midrule
\multicolumn{9}{l}{\textit{General time-series models}} \\
Transformer$^\dagger$ \cite{moreno-pino_rough_2024} & OOM & $84.3{\scriptstyle\pm6.3}$ & $49.1{\scriptstyle\pm2.5}$ & $\underline{40.5{\scriptstyle\pm6.3}}$ & $70.5{\scriptstyle\pm0.1}$ & $50.5{\scriptstyle\pm3.0}$ & $59.0$ & $10.3$ \\
RFormer  \cite{moreno-pino_rough_2024} & $90.3{\scriptstyle\pm0.0}$  & $81.2{\scriptstyle\pm2.0}$ & $52.3{\scriptstyle\pm3.0}$ & $34.7{\scriptstyle\pm4.0}$ & $72.5{\scriptstyle\pm0.0}$ & $55.8{\scriptstyle\pm6.0}$ & $64.5$ & $8.0$ \\
TS2Vec$^\ast$  \cite{yue2022ts2vecuniversalrepresentationtime}    & $84.7$ & $81.2$ & $57.8$ & $30.8$ & $68.3$ & $51.0$ & $62.3$ & $9.3$ \\
MOMENT$^\ast$ \cite{goswami2024momentfamilyopentimeseries}     & $80.9$ & $84.0$ & $47.8$ & $\mathit{35.7}$ & $72.2$ & $50.0$ & $61.8$ & $10.2$ \\
TSPulse$^{\dagger\ast}$ \cite{ekambaram2026tspulsetinypretrainedmodels} & --- & $83.6$ & $51.1$ & $24.7$ & $70.2$ & $58.0$ & $57.5$ & $11.7$ \\
\midrule
\multicolumn{9}{l}{\textit{SSM / recurrent models}} \\
LRU \cite{orvieto_resurrecting_2023}         & $87.8{\scriptstyle\pm2.8}$  & $82.6{\scriptstyle\pm3.4}$ & $51.2{\scriptstyle\pm3.6}$ & $21.5{\scriptstyle\pm2.1}$ & $\underline{78.4{\scriptstyle\pm6.7}}$& $48.4{\scriptstyle\pm5.0}$ & $61.7$ & $9.7$ \\
S5 \cite{smith_simplified_2023}         & $81.1{\scriptstyle\pm3.7}$  & $\mathbf{89.9{\scriptstyle\pm4.6}}$ & $50.5{\scriptstyle\pm2.6}$ & $24.1{\scriptstyle\pm4.3}$ & $\mathit{77.7{\scriptstyle\pm5.5}}$ & $47.7{\scriptstyle\pm5.5}$ & $61.8$ & $9.3$ \\
Mamba \cite{gu_mamba_2024} & $70.9{\scriptstyle\pm15.8}$ & $80.7{\scriptstyle\pm1.4}$ & $48.2{\scriptstyle\pm3.9}$ & $27.9{\scriptstyle\pm4.5}$ & $76.2{\scriptstyle\pm3.8}$ & $47.7{\scriptstyle\pm4.5}$ & $58.6$ & $12.5$ \\
LinOSS-IMEX \cite{rusch_oscillatory_2025} & $80.0{\scriptstyle\pm2.7}$  & $87.5{\scriptstyle\pm4.0}$ & $\mathbf{58.9{\scriptstyle\pm8.1}}$ & $29.9{\scriptstyle\pm1.0}$ & $75.5{\scriptstyle\pm4.3}$ & $57.9{\scriptstyle\pm5.3}$ & $65.0$ & $6.7$ \\
LinOSS-IM \cite{rusch_oscillatory_2025} & $\underline{95.0{\scriptstyle\pm4.4}}$ & $87.8{\scriptstyle\pm2.6}$ & $58.2{\scriptstyle\pm6.9}$ & $29.9{\scriptstyle\pm0.6}$ & $75.8{\scriptstyle\pm3.7}$ & $\underline{60.0{\scriptstyle\pm7.5}}$ & $\mathit{67.8}$ & $\mathit{4.3}$ \\
D-LinOSS \cite{boyer_learning_2025} & $\mathit{93.9{\scriptstyle\pm3.2}}$  & $\underline{88.9{\scriptstyle\pm3.0}}$ & $\underline{58.6{\scriptstyle\pm2.3}}$ & $29.9{\scriptstyle\pm0.6}$ & $75.8{\scriptstyle\pm4.9}$ & $\mathbf{61.1{\scriptstyle\pm2.0}}$ & $\underline{68.0}$ & $\underline{3.7}$ \\
PD-SSM \cite{terzic2025structuredsparsetransitionmatrices}     & $90.0{\scriptstyle\pm5.7}$  & $80.9{\scriptstyle\pm2.0}$ & $56.1{\scriptstyle\pm8.6}$ & $34.7{\scriptstyle\pm4.0}$ & $\mathbf{80.0{\scriptstyle\pm2.6}}$ & $\underline{60.0{\scriptstyle\pm3.7}}$ & $67.0$ & $5.2$ \\
\midrule
LogSig-SSM & $\mathbf{95.3{\scriptstyle\pm3.6}}$ & $\mathit{88.4{\scriptstyle\pm1.9}}$ & $\mathit{58.4{\scriptstyle\pm2.9}}$ & $\mathbf{51.4{\scriptstyle\pm6.9}}$ & $76.9{\scriptstyle\pm9.7}$ & $59.9{\scriptstyle\pm3.4}$ & $\mathbf{71.7}$ & $\mathbf{2.7}$ \\
\bottomrule
\end{tabular}
\end{table*}

\begin{table}[h]
\centering
\footnotesize
\caption{Efficiency on \textsc{EigenWorms} ($T=17{,}984$): GPU memory (MB), wall-clock time per 1000 training steps (s) on one RTX~4090, and test accuracy (\%). Full results in Appendix~\ref{app:memory_runtime}.}
\label{tab:worms_efficiency}
\setlength{\tabcolsep}{4pt}
\begin{tabular}{lcccccccc}
\toprule
 & NCDE & NRDE & Log-NCDE & Mamba & S5 & LinOSS-IM & D-LinOSS & \textbf{LogSig-SSM} \\
\midrule
Mem.\ (MB)          & 2484 & 2506 & 2510 & 13486 & 6646 & 3488 & 3488 & \textbf{362} \\
Time (s/1k steps)   & 24595 & 5386 & 1956 & 122 & 31 & 14 & 14 & \textbf{4} \\
Accuracy (\%)       & 75.0 & 83.9 & 85.6 & 70.9 & 81.1 & 95.0 & 93.9 & \textbf{95.3} \\
\bottomrule
\end{tabular}
\end{table}
 
\begin{table*}[t]
\centering
\footnotesize
\caption{Results on regression, forecasting, and irregular-sampling benchmarks.
Left: PPG-DaLiA heart-rate regression, MSE\,$\times10^{-2}$ ($\downarrow$), mean\,$\pm$\,std over five runs~\cite{attila_reiss_ppg-dalia_2019}.
Centre: weather forecasting MAE ($\downarrow$), predict 30 days of hourly data from the preceding 30 days; Informer$^{\dagger}$ uses canonical self-attention.
Right: PhysioNet Sepsis AUROC ($\uparrow$), $10.3\%$ of values observed; OI\,=\,with observation indicators~\cite{PhysioNet-challenge-2019-1.0.0}.
Bold\,=\,best, underline\,=\,second best per sub-table.}
\label{tab:three_benchmarks}
\setlength{\tabcolsep}{5pt}
\renewcommand{\arraystretch}{0.95}
\begin{tabular}{lc@{\hspace{14pt}}lc@{\hspace{14pt}}lcc}
\toprule
\multicolumn{2}{c}{\textbf{Regression (PPG-DaLiA)}} &
\multicolumn{2}{c}{\textbf{Forecasting (Weather)}} &
\multicolumn{3}{c}{\textbf{Irregular Sampling (PhysioNet Sepsis)}} \\
\cmidrule(lr){1-2}\cmidrule(lr){3-4}\cmidrule(lr){5-7}
Model & MSE\,$\!\times\!10^{-2}$ & Model & MAE & Method & OI & No OI \\
\midrule
NRDE         & $9.90{\scriptstyle\pm0.97}$             & Informer             & $0.731$              & GRU-$\Delta t$ & $87.8{\scriptstyle\pm0.6}$ & $84.0{\scriptstyle\pm0.7}$ \\
NCDE         & $13.54{\scriptstyle\pm0.69}$            & Informer$^{\dagger}$ & $0.741$              & GRU-D          & $87.1{\scriptstyle\pm2.2}$ & $85.0{\scriptstyle\pm1.3}$ \\
Log-NCDE     & $9.56{\scriptstyle\pm0.59}$             & LogTrans             & $0.773$              & GRU-ODE        & $85.2{\scriptstyle\pm1.0}$ & $77.1{\scriptstyle\pm2.4}$ \\
LRU          & $12.17{\scriptstyle\pm0.49}$            & Reformer             & $1.575$              & ODE-RNN        & $87.4{\scriptstyle\pm1.6}$ & $83.3{\scriptstyle\pm2.0}$ \\
S5           & $12.63{\scriptstyle\pm1.25}$            & LSTMa                & $1.109$              & Neural CDE     & $88.0{\scriptstyle\pm0.6}$ & $77.6{\scriptstyle\pm0.9}$ \\
S6           & $12.88{\scriptstyle\pm2.05}$            & LSTnet               & $0.757$              & ANCDE          & $90.0{\scriptstyle\pm0.2}$ & $82.3{\scriptstyle\pm0.3}$ \\
Mamba        & $10.65{\scriptstyle\pm2.20}$            & S4                   & $0.578$              & Neural LSDE    & $\underline{90.9{\scriptstyle\pm0.4}}$ & $87.9{\scriptstyle\pm0.8}$ \\
LinOSS-IMEX  & $7.50{\scriptstyle\pm0.46}$             & LinOSS-IMEX          & $0.508$              & Neural LNSDE   & $\mathbf{91.1{\scriptstyle\pm0.2}}$ & $\underline{88.1{\scriptstyle\pm0.2}}$ \\
LinOSS-IM    & $6.40{\scriptstyle\pm0.23}$             & LinOSS-IM            & $0.528$              & Neural GSDE    & $\underline{90.9{\scriptstyle\pm0.1}}$ & $\mathbf{88.4{\scriptstyle\pm0.2}}$ \\
D-LinOSS     & $\underline{6.16{\scriptstyle\pm0.73}}$ & D-LinOSS             & $\underline{0.486}$  &                &                            &                            \\
\midrule
LogSig-SSM & $\mathbf{4.46{\scriptstyle\pm0.60}}$ & LogSig-SSM & $\mathbf{0.432}$ & LogSig-SSM & $90.4{\scriptstyle\pm0.4}$ & $87.4{\scriptstyle\pm0.8}$ \\
\bottomrule
\end{tabular}
\end{table*}
\subsection{PPG-DaLiA time-series regression}
\label{sec:ppg_results}
\textsc{PPG-DaLiA}~\cite{attila_reiss_ppg-dalia_2019} is a multimodal physiological regression benchmark for heart-rate estimation from wrist-worn sensors. It contains recordings from 15 subjects during daily activities, each with approximately 150 minutes of data sampled at up to 128\,Hz. Following the Log-NCDE protocol~\cite{walker_log_2024}, recordings are converted into long sequences via a sliding window of length $49{,}920$ and stride $4{,}992$.
LogSig-SSM attains MSE $4.46\!\times\!10^{-2}$, a 28\% improvement over D-LinOSS~\cite{boyer_learning_2025} ($6.16\!\times\!10^{-2}$) and 58\% over Mamba~\cite{gu_mamba_2024} ($10.65\!\times\!10^{-2}$), the strongest relative gain across all four benchmarks (Table~\ref{tab:three_benchmarks}, left). On these ultra-long sequences ($49{,}920$ time steps), log-signature tokenisation directly alleviates the long-context degradation that limits standard selective SSMs, making training feasible in regimes where several baselines are memory-constrained.
 
\subsection{Weather Forecasting}
\label{sec:weather_results}
 
Following the precedent set by prior state-space models \cite{gu_efficiently_2022, rusch_oscillatory_2025}, we evaluate LogSig-SSM's capacity to serve as a general-purpose sequence-to-sequence model on long-horizon forecasting tasks. We focus on the challenging weather prediction benchmark introduced by \citet{zhou2021informerefficienttransformerlong}, specifically targeting the hardest configuration: predicting the future 720 timesteps (one month of hourly multivariate climatological data) based directly on the preceding 720 timesteps. Without incorporating any domain-specific inductive biases or specialised forecasting architectures, LogSig-SSM achieves an MAE of $0.432$ (Table~\ref{tab:three_benchmarks}), reducing the best prior result from D-LinOSS ($0.486$) by 11\% and outperforming all other baselines by a larger margin. These results demonstrate that multi-scale log-signature tokenisation effectively generalises to standard long-horizon forecasting settings.

\subsection{Sepsis Prediction on PhysioNet (Irregular Sampling)}
\label{sec:physionet_results}

We evaluate on the PhysioNet 2019 early sepsis prediction challenge~\cite{PhysioNet-challenge-2019-1.0.0,doi:10.1161/01.CIR.101.23.e215}. Comprising 40,335 variable-length ICU stays with 34 features, this dataset is extremely sparse (only $10.3\%$ observed values), providing a direct test of the irregular-sampling robustness established in Section~\ref{sec:th_sampling}. Predicting sepsis occurrence based on the first 72 hours, we evaluate AUROC both with and without appended observation indicators (OI) (Table~\ref{tab:three_benchmarks}, right). 

LogSig-SSM achieves an AUROC of $90.4/87.4$ (with/without OI). It performs competitively with purpose-built irregular-time models (Neural LSDE, LNSDE, GSDE~\cite{oh2025stableneuralstochasticdifferential}) and outperforms Neural CDE ($88.0/77.6$) \cite{kidger_neural_2020}, ANCDE ($90.0/82.3$) \cite{ANCDE}, and all GRU variants. This empirically confirms that within-window reparameterisation invariance yields practical robustness under extreme sparsity.
\subsection{Ablation Studies}
\label{sec:ablations_main}
Table~\ref{tab:ablation_ladder} isolates the contributions of LogSig-SSM along two axes, the tokeniser and the backbone, with every variant tuned under the same protocol and budget. Read top to bottom, the descriptor rows decompose the $13.1$\pp\ gap between raw Mamba ($58.6\%$) and LogSig-SSM ($71.7\%$). Compression alone contributes little, $+0.3$\pp\ with Catch24 statistics that discard the path and $+1.5$\pp\ with a learned per-channel descriptor that keeps it. Linear channel mixing adds $+2.8$\pp\ ($62.9\%$), multi-scale spectral content another $+4.7$\pp\ ($67.6\%$), multi-scale increments, the depth-one log-signature, a further $+1.9$\pp\ ($69.5\%$), and the second-order terms the final step to $71.7\%$. This last term, the cross-channel geometry, is the one that neither linear projections nor spectral summaries express. With the tokeniser fixed, S5 and Transformer backbones both far exceed raw Mamba, so the tokens are useful regardless of backbone, while selective recurrence still leads on each of the six datasets.

\paragraph{Truncation depth.} The depth-one row of Table~\ref{tab:ablation_ladder} is a matched-budget arm that re-searches the tokeniser with the depth fixed at one (Appendix~\ref{app:depth_study}). Forcing depth one costs $2.3$ points on average, almost entirely on the three datasets with few channels ($7.2$, $4.3$ and $2.8$ on \textsc{EigenWorms}, SCP1 and SCP2). Wavelet scattering makes the content of the second-order terms visible: on those three datasets the depth-one arm is level with Wst-Mamba ($-1.0$, $-1.8$, $+2.5$ points) while depth two leads it by $6.2$, $2.5$ and $5.3$, so the margin over the strongest alternative descriptor is exactly the cross-channel geometry of Proposition~\ref{prop:th_mixing_tradeoff}. Signature tokens of the same depth and configuration cost $3.1$ points on average over the same three datasets and roughly double tokenisation time (Appendix~\ref{app:sig_vs_logsig}).

\paragraph{Search-free tokeniser.} The tokeniser still improves performance without hyperparameter tuning. Keeping the Mamba backbone of the reported LogSig-SSM configuration and replacing the searched tokeniser hyperparameters by a rule based on the sequence length alone ($s=T/100$, lengths $[T/25, T/100]$, depths $[2,1]$), LogSig-SSM still reaches $64.0\%$ average accuracy, above S5 ($61.8$), LRU ($61.7$) and Mamba ($58.6$), and searching the tokeniser adds the remaining $7.7$ points (Appendix~\ref{app:hyperparams}).

\begin{table}[h]
\centering
\scriptsize
\caption{UEA ablations, average accuracy over the six datasets (\%\,$\uparrow$). All variants share the 200-trial Optuna protocol and budget of LogSig-SSM, and its window positions and token dimension where applicable; $\Delta$ is relative to raw Mamba. Per-dataset results in Appendix~\ref{app:ablations_uea}, Table~\ref{tab:unified_ablation}.}
\label{tab:ablation_ladder}
\setlength{\tabcolsep}{7pt}
\renewcommand{\arraystretch}{0.95}
\begin{tabularx}{\textwidth}{llcX}
\toprule
Configuration & Avg.\ acc.\ (\%) & $\Delta$ (pp) & What it adds \\
\midrule
\multicolumn{4}{l}{\textit{Window descriptor (all rows compress the sequence), Mamba backbone fixed}}\\
Mamba (raw input)                              & \accbar{0.54cm}{58.6}                    & --               & no compression, no channel mixing \\
Catch24 statistics \cite{lubba_catch22_2019}   & \accbar{0.59cm}{58.9}                    & $+0.3$           & path discarded, summary statistics only \\
Depthwise convolution                          & \accbar{0.77cm}{60.1}                    & $+1.5$           & channels kept separate \\
Standard convolution                           & \accbar{1.19cm}{62.9}                    & $+4.3$           & first-order (linear) channel mixing \\
Wavelet scattering                             & \accbar{1.89cm}{67.6}                    & $+9.0$           & multi-scale spectral content, no cross-channel geometry \\
Log-signature, depth one            & \accbar{2.18cm}{69.5}                    & $+10.9$          & multi-scale increments, no cross-channel content \\
Log-signature, depth two (\textbf{LogSig-SSM}) & \accbar[green!45]{2.51cm}{\textbf{71.7}} & $\mathbf{+13.1}$ & second-order terms: cross-channel geometry, L\'evy areas (Prop.~\ref{prop:th_mixing_tradeoff}) \\
\midrule
\multicolumn{4}{l}{\textit{Backbone, log-signature tokeniser fixed}}\\
LogSig-S5                                      & \accbar{1.20cm}{63.0}                    & $+4.4$           & LTI recurrence, no selectivity \\
LogSig-Transformer                             & \accbar{1.56cm}{65.4}                    & $+6.8$           & attention, no recurrence \\
\bottomrule
\end{tabularx}
\end{table}

\section{Conclusion}
\label{sec:conclusion}
We introduce LogSig-SSM, which combines multi-scale windowed log-signature tokenisation with a selective SSM backbone for long-horizon time-series modelling. By replacing pointwise inputs with interval descriptors, LogSig-SSM compresses long, irregularly sampled sequences into a shorter token stream while injecting higher-order cross-channel information through log-signature features, enabling scan-efficient long-range propagation while increasing single-block multivariate expressivity.

Beyond its discrete-time formulation, LogSig-SSM extends the linear CDE view of selective SSMs~\cite{cirone_theoretical_2025} to log-signature-derived driving paths, yielding an NRDE-style interval-driven form in which selectivity allocates internal time across the compressed sequence in a content-dependent manner. Empirically, LogSig-SSM achieves the best or competitive performance across the four benchmarks above while substantially reducing training time and GPU memory relative to Mamba-style baselines. Ablation studies confirm that the gains arise from the geometric content of log-signature tokens, with selective recurrence providing additional margin on tasks requiring content-dependent aggregation.

A limitation for LogSig-SSM is that for genuinely single-channel paths, log-signature tokens reduce to path increments and the higher-order cross-channel structure becomes degenerate, although augmenting the path with time as an additional channel restores non-trivial structure. The token width grows as $O(d_x^{m}/m)$, so tokenisation time and the size of the input embedding increase rapidly with the number of channels. Finally, window lengths and stride are still selected per dataset (Appendix~\ref{app:hyperparams}). Learned or adaptive window selection is a natural direction for reducing this overhead.

\begin{ack}
We thank Helena Latorre Moreno for her feedback on the mathematical proofs and Foucault de Jerphanion for his help in the early stages of this work. This work was supported by the UKRI AI Centre for Doctoral Training in Digital Healthcare (grant number EP/Y030974/1).
\end{ack}

\newpage
\bibliographystyle{plainnat}
\bibliography{references}

\appendix
\appendix
\newpage
\section{Additional Derivations}
\label{app:derivations}

\subsection{Derivation of the Linear CDE Interpretation of the Selective Scan}
\label{app:linear_cde_derivation}

We provide the derivation underlying the continuous-time interpretation stated in Section~\ref{sec:sig_ssm_continuous}, following the standard small-step approximation used in selective SSM analyses \cite{cirone_theoretical_2025} and replacing the pointwise input by the embedded token path $\widetilde U$.

In the continuous-time formulation we take $W_B\in\R^{N\times d}$ and identify the discrete-time input projection by $s_B(\widetilde U):=W_B\,\widetilde U\in\R^{N}$. Fix a partition $0=r_0<\cdots<r_{T'}=1$ and extend the embedded token sequence $\{\widetilde U_\tau\}_{\tau=1}^{T'}$ to a ZOH path $\widetilde U:[0,1]\to\R^{d}$ by \eqref{eq:sigssm_U_path}. Introduce a base step $\delta>0$ and rescale
\begin{equation}
\widetilde\alpha_i := \alpha_i/\delta,
\qquad
\widetilde\beta_i := \beta_i/\delta,
\qquad i=1,\dots,d.
\end{equation}
Following \cite{cirone_theoretical_2025}, we further approximate the softplus gate by a ReLU,
$\sigma(x)=\mathrm{ReLU}(x)\approx \log(1+e^{x})$, which yields for small $\delta$,
\begin{equation}
\Delta_i(\widetilde U)
=
\softplus(\alpha_i^\top \widetilde U+\beta_i)
=
\softplus\!\big(\delta(\widetilde\alpha_i^\top \widetilde U+\widetilde\beta_i)\big)
\;\approx\;
\delta\,\sigma(\widetilde\alpha_i^\top \widetilde U+\widetilde\beta_i).
\label{eq:sigssm_delta_relu_approx_app}
\end{equation}
Consequently, as $\delta\to 0$, a first-order expansion of the matrix exponential gives
\begin{equation}
\overline A_i(\widetilde U)
=
\exp(\Delta_i(\widetilde U)A_i)
=
\exp\!\big(\delta\,\sigma(\widetilde\alpha_i^\top \widetilde U+\widetilde\beta_i)\,A_i\big)
\;\approx\;
I+\delta\,\sigma(\widetilde\alpha_i^\top \widetilde U+\widetilde\beta_i)\,A_i .
\label{eq:sigssm_expansion_Abar_app}
\end{equation}

For the injection, we use the standard small-step approximation consistent with the Mamba-style selective scan:
\begin{equation}
\overline B_{i}(\widetilde U_\tau)\,\widetilde U_\tau^i
\;\approx\;
s_B(\widetilde U_\tau)\,\Delta_i(\widetilde U_\tau)\,\widetilde U_\tau^i
\;\approx\;
s_B(\widetilde U_\tau)\,\sigma(\widetilde\alpha_i^\top \widetilde U_\tau+\widetilde\beta_i)\,\delta\,\widetilde U_\tau^i.
\label{eq:sigssm_Bbar_approx_app}
\end{equation}

Substituting \eqref{eq:sigssm_expansion_Abar_app}--\eqref{eq:sigssm_Bbar_approx_app} into the selective recurrence \eqref{eq:sigssm_ssm_form} gives the Euler-type update
\begin{equation}
H_{i,\tau}
=
H_{i,\tau-1}
+
A_iH_{i,\tau-1}\,\sigma(\widetilde\alpha_i^\top \widetilde U_\tau+\widetilde\beta_i)\,\delta
+
s_B(\widetilde U_\tau)\,\widetilde U_\tau^i\,\sigma(\widetilde\alpha_i^\top \widetilde U_\tau+\widetilde\beta_i)\,\delta,
\label{eq:sigssm_euler_like_app}
\end{equation}
which is the forward Euler discretisation of the ODE
\begin{equation}
\dot H^{\widetilde U}_{i;t}
=
A_i H^{\widetilde U}_{i;t}\,\sigma(\widetilde\alpha_i^\top \widetilde U_t+\widetilde\beta_i)
+
s_B(\widetilde U_t)\,\widetilde U_t^i\,\sigma(\widetilde\alpha_i^\top \widetilde U_t+\widetilde\beta_i),
\qquad i=1,\dots,d.
\label{eq:sigssm_ode_limit_app}
\end{equation}

Define the drivers
\begin{equation}
\omega^{\widetilde U}_{i;t}
:=
\int_0^t \sigma(\widetilde\alpha_i^\top \widetilde U_s+\widetilde\beta_i)\,ds
\in\R,
\qquad
\xi^{\widetilde U}_{i;t}
:=
\int_0^t \widetilde U_s\,\widetilde U_s^i\,\sigma(\widetilde\alpha_i^\top \widetilde U_s+\widetilde\beta_i)\,ds
\in\R^{d}.
\label{eq:sigssm_drivers_mamba_style_app}
\end{equation}
Then \eqref{eq:sigssm_ode_limit_app} can be rewritten in linear CDE form as
\begin{equation}
dH^{\widetilde U}_{i;t}
=
A_i H^{\widetilde U}_{i;t}\,d\omega^{\widetilde U}_{i;t}
+
W_B\,d\xi^{\widetilde U}_{i;t},
\qquad
H^{\widetilde U}_{i;0}=h_{i,0},
\qquad i=1,\dots,d,
\label{eq:sigssm_linear_cde_app}
\end{equation}
which matches \eqref{eq:sigssm_linear_cde} in the main text.

For completeness, note that if $\widetilde U(t)$ is held constant on $[r_{\tau-1},r_{\tau})$ and the system is evolved for an internal increment $\Delta_{i,\tau}$, the exact ZOH discretisation gives:
\begin{equation}
H_{i,\tau}
=
\exp(\Delta_{i,\tau}A_i)\,H_{i,\tau-1}
+
(\Delta_{i,\tau}A_i)^{-1}\!\big(\exp(\Delta_{i,\tau}A_i)-I\big)\,
\big(\Delta_{i,\tau}\,s_B(\widetilde U_\tau)\big)\,\widetilde U_\tau^i.
\label{eq:sigssm_exact_step_backlink_app}
\end{equation}
In particular, taking $\Delta_{i,\tau}=\Delta_i(\widetilde U_\tau)$ yields the Mamba-style selective recurrence
\[
H_{i,\tau}=\overline A_{i,\tau}H_{i,\tau-1}+\overline B_{i,\tau}\,\widetilde U_\tau^i,
\]
with $\overline A_{i,\tau}=\exp(\Delta_{i,\tau}A_i)$ and
$\overline B_{i,\tau}=(\Delta_{i,\tau}A_i)^{-1}\!\big(\exp(\Delta_{i,\tau}A_i)-I\big)\,(\Delta_{i,\tau}s_B(\widetilde U_\tau))$.

\newpage

\subsection{Derivation of the NRDE Interval Form}
\label{app:nrde_interval_derivation}
 
We derive the interval-driven dynamics stated in Section~\ref{sec:sig_ssm_continuous}, starting from the linear CDE \eqref{eq:sigssm_linear_cde} under the ZOH extension \eqref{eq:sigssm_U_path}.
 
Since $\widetilde U(t)=\widetilde U_\tau$ is constant on $[r_{\tau-1},r_\tau)$, the drivers defined in \eqref{eq:sigssm_drivers_mamba_style} are affine on each interval with constant derivatives:
\begin{equation}
\frac{d}{dt}\omega^{\widetilde U}_{i;t} = \sigma(\widetilde\alpha_i^\top\widetilde U_\tau+\widetilde\beta_i),
\qquad
\frac{d}{dt}\xi^{\widetilde U}_{i;t} = \widetilde U_\tau\,\widetilde U_\tau^i\,\sigma(\widetilde\alpha_i^\top\widetilde U_\tau+\widetilde\beta_i),
\qquad t\in[r_{\tau-1},r_\tau).
\end{equation}
Substituting into \eqref{eq:sigssm_linear_cde} gives the ODE on the interval
\begin{equation}
\frac{d}{dt}H^{\widetilde U}_{i;t}
= A_i H^{\widetilde U}_{i;t}\,\sigma(\widetilde\alpha_i^\top\widetilde U_\tau+\widetilde\beta_i)
+ W_B\,\widetilde U_\tau\,\widetilde U_\tau^i\,\sigma(\widetilde\alpha_i^\top\widetilde U_\tau+\widetilde\beta_i).
\end{equation}
Defining the interval increments
\begin{equation}
\Lambda^{\widetilde U}_{i,\tau}
:= \begin{pmatrix}\Delta\omega^{\widetilde U}_{i,\tau}\\\Delta\xi^{\widetilde U}_{i,\tau}\end{pmatrix}
= \begin{pmatrix}\sigma(\widetilde\alpha_i^\top\widetilde U_\tau+\widetilde\beta_i)\,\Delta r_\tau\\
\widetilde U_\tau\,\widetilde U_\tau^i\,\sigma(\widetilde\alpha_i^\top\widetilde U_\tau+\widetilde\beta_i)\,\Delta r_\tau\end{pmatrix},
\qquad \Delta r_\tau := r_\tau - r_{\tau-1},
\end{equation}
and writing the ODE in the form $\frac{d}{dt}H^{\widetilde U}_{i;t} = [A_i H^{\widetilde U}_{i;t},\,W_B]\,\frac{1}{\Delta r_\tau}\Lambda^{\widetilde U}_{i,\tau}$, integration over $[r_{\tau-1},t]$ yields
\begin{equation}
H^{\widetilde U}_{i;t}
= H^{\widetilde U}_{i;r_{\tau-1}}
+ \int_{r_{\tau-1}}^{t}
\bigl[A_i H^{\widetilde U}_{i;s},\;W_B\bigr]\,
\frac{1}{\Delta r_\tau}\Lambda^{\widetilde U}_{i,\tau}\,ds,
\qquad t\in(r_{\tau-1},r_\tau),
\end{equation}
which is equation \eqref{eq:sigssm_nrde_interval_integral_33} in the main text. The term $\Delta\omega^{\widetilde U}_{i,\tau}=\sigma(\widetilde\alpha_i^\top\widetilde U_\tau+\widetilde\beta_i)\Delta r_\tau$ depends on the input content $\widetilde U_\tau$, making the internal-time allocation content-dependent rather than fixed as in standard NRDEs.

\subsection{Proof of Proposition~\ref{prop:th_mixing_tradeoff}}
\label{app:expressivity_proof}

\paragraph{Closure of a single Mamba block.}
For a generic input path $\gamma$, the uniform closure of a single
diagonal Mamba block (Cor.~4.4 of \citet{cirone_theoretical_2025})
takes the form
\begin{equation}
\mathcal{F}^{(1)}_{\mathrm{Mamba}}[\mathcal{X}] :=
\overline{\Bigl\{
(\gamma,t)\mapsto
\textstyle\sum_{i=1}^{d}\psi_i(\omega^{\gamma,i}_t)
+ \sum_{i=1}^{d}\int_0^t
\phi_i(\omega^{\gamma,i}_t - \omega^{\gamma,i}_s)\,d\xi^{\gamma,i}_s
\Bigr\}_{\gamma\in\mathcal{X}}},
\label{eq:mamba_closure_app}
\end{equation}
with continuous $\psi_i,\phi_i:\mathbb{R}\to\mathbb{R}$ (notation
as in \citet{cirone_theoretical_2025}). Each summand depends on
$\gamma$ only through scalar gate paths $\omega^{\gamma,i}$ and
$ds$-driven integrators $\xi^{\gamma,i}$; in particular, no
$d\gamma$-driven iterated integral appears.

\paragraph{Proof of (i).}
By construction, the token $U_\tau\in\mathbb{R}^{D_{\mathrm{tot}}}$
contains every log-signature coordinate of $\mathcal{W}_k(X)_\tau$
up to depth $m_k\geq m$ (Eq.~\ref{eq:sigssm_token}). After fixing
a Lie basis on the truncated free Lie algebra, the assumed factorisation
of $F$ gives $F(X) = \tilde f(U_\tau(X))$ for some continuous
$\tilde f:\mathbb{R}^{D_{\mathrm{tot}}}\to\mathbb{R}$. The map
$X\mapsto U_\tau(X)$ is continuous and $\mathcal{X}$ is compact, so
the image
$\mathcal{U} := \{U_\tau(X) : X\in\mathcal{X}\}\subset
\mathbb{R}^{D_{\mathrm{tot}}}$ is compact. Take state dimension
$d\geq D_{\mathrm{tot}}$ and $W_{\mathrm{in}}\in\mathbb{R}^{d\times
D_{\mathrm{tot}}}$ injective; the image
$\widetilde{\mathcal{U}} := W_{\mathrm{in}}\mathcal{U}$ is also
compact.

We adopt the small-step parameterisation
$\overline B_{i,\tau} = \Delta_{i,\tau}\, s_B(\widetilde U_\tau)$,
implemented in Mamba \citep{gu_mamba_2024} and used in the linear-CDE
bridge of \citet{cirone_theoretical_2025} (cf.\ Section~\ref{sec:sig_ssm_continuous}).
Here $s_B(\widetilde U) = W_B\widetilde U$ and $s_C(\widetilde U)
= W_C\widetilde U$ are linear in $\widetilde U$ with
learnable parameters $W_B, W_C$, while $\overline B_{i,\tau}$ and
$s_C(\widetilde U_\tau)$ vary in $\tau$ through their input
dependence. Since softplus is continuous and strictly positive and
$\widetilde{\mathcal{U}}$ is compact, $\Delta_{i,\tau} =
\mathrm{softplus}(\alpha_i^\top\widetilde U_\tau + \beta_i)$ is
uniformly bounded above and below by constants
$\Delta_{\max}\geq\Delta_{\min}>0$, so $\overline B_{i,\tau}$ is
uniformly bounded. Writing
\[
A_i = -\operatorname{diag}(\lambda_{i,1},\dots,\lambda_{i,N}),
\qquad \lambda_{i,k}\geq\lambda_{\min}>0,
\]
yields
\[
\overline A_{i,\tau}
= \operatorname{diag}\!\big(e^{-\lambda_{i,1}\Delta_{i,\tau}},
\dots,e^{-\lambda_{i,N}\Delta_{i,\tau}}\big),
\qquad
\|\overline A_{i,\tau}\|
\leq e^{-\lambda_{\min}\Delta_{\min}}
\xrightarrow{\lambda_{\min}\to\infty} 0,
\]
uniformly on $\widetilde{\mathcal{U}}$.

The recurrence
$H_{i,\tau} = \overline A_{i,\tau} H_{i,\tau-1}
+ \overline B_{i,\tau}\,\widetilde U_\tau^i$
keeps $H$ uniformly bounded; taking $\lambda_{\min}\to\infty$, 
\[
H_{i,\tau} \;\to\; \Delta_{i,\tau}\,W_B\widetilde U_\tau\,
\widetilde U_\tau^i,
\]
being a continuous function of $\widetilde U_\tau$ alone. The limiting
per-anchor readout vector
$\{y_{i,\tau} := s_C(\widetilde U_\tau)\, H_{i,\tau}\}_{i=1}^d$
is then a smooth map
$g:\widetilde{\mathcal{U}}\to\mathbb{R}^{d}$, polynomial in
$\widetilde U_\tau$ modulated by the softplus gate.

For $d\geq D_{\mathrm{tot}}$, $g$ is a smooth map between spaces of
matching ambient dimension. There exist parameters
$(W_B, W_C, \alpha, \beta, A)$ such that $g$ is injective on the compact set $\widetilde{\mathcal{U}}$.
Combined with injectivity of $W_{\mathrm{in}}$, the composed map
$U_\tau\mapsto\{y_{i,\tau}\}_{i=1}^d$ is injective on $\mathcal{U}$.
Standard MLP universality on the compact image then approximates
$\tilde f$ uniformly, hence $F$, to error $\varepsilon$.

\paragraph{Proof of (ii).}
By Cor.~4.4 of \citet{cirone_theoretical_2025}, every feature in
$\mathcal{F}^{(1)}_{\mathrm{Mamba}}[\mathcal{X}]$ factorises through
scalar gates $\omega^{X,i}$ and $ds$-driven integrators $\xi^{X,i}$
as in \eqref{eq:mamba_closure_app}; in Mamba, $d\xi^{X,i}_s = X_s 
X^i_s\,\sigma(\widetilde\alpha_i\!\cdot\! X_s + \widetilde\beta_i)\,ds$
is $ds$-driven, depending on $X$ only through pointwise values $X_s$.
A pointwise post-scan MLP composed with such block outputs produces
only continuous functions of $\mathcal{F}^{(1)}_{\mathrm{Mamba}}[\mathcal{X}]$;
since closure features are $ds$-driven and scalar-gated, no $dX$-driven
iterated integral lies in this composed family. The L\'evy area
\[
A^{12}_t = \tfrac{1}{2}\!\int_0^t \big(X^1_s\,dX^2_s
- X^2_s\,dX^1_s\big)
\]
is a $dX$-driven iterated integral, antisymmetric under
$1\leftrightarrow 2$; no finite combination of $ds$-driven features
in the closure can reproduce it, so
$A^{12}\notin\mathcal{F}^{(1)}_{\mathrm{Mamba}}[\mathcal{X}]$.

\paragraph{Proof of (iii).}
Direct from Sec.~C.2.1 of \citet{cirone_theoretical_2025}: chained
diagonal blocks recover progressively richer iterated-integral
features as the chain depth grows, populating the state at each
anchor with the cross-channel structure that is unavailable at
chain depth one. The asymptotic universality of
\citet{cirone_theoretical_2025} (Prop.~4.5) requires chain depth
growing with the maximum iterated-integral depth being recovered,
whereas LogSig-SSM realises any such feature in a single block
provided $\exists\, k: m_k\geq m$.
\hfill$\square$

\newpage
\section{Memory and Runtime}
\label{app:memory_runtime}
Table~\ref{tab:uea_compute} summarises the compute requirements of each model on the UEA classification benchmarks of Section~\ref{sec:uea_results}. Baseline numbers are taken from \cite{walker_log_2024,boyer_learning_2025}, and we report parameter count $|\theta|$, GPU memory usage (MB), and wall-clock time per 1000 training steps (s), all measured on a single NVIDIA RTX 4090.

LogSig-SSM is the most compute-efficient approach. It has the lowest memory usage on all six datasets, often an order of magnitude below Mamba, and it is the fastest (or joint fastest) model on all six datasets, with large gains over the continuous-time baselines. These gains follow from the architecture: log-signature tokenisation shortens the effective sequence before the selective scan, reducing both the number of recurrent updates and the activations stored for backpropagation. On the two widest datasets (Heartbeat, 61 channels; MotorImagery, 64 channels) the depth-two tokens are 1{,}891- and 4{,}160-dimensional, which raises the parameter count of the input embedding but leaves the footprint at 54\,MB and 78\,MB.

Tokenisation is a one-time offline computation and is not part of the training graph. Table~\ref{tab:uea_sigtime} reports its cost with \texttt{iisignature} \cite{reizenstein2018iisignaturelibraryefficientcalculation} on CPU, with the Lie basis built once per dataset rather than once per series (the emitted features are identical): all six datasets are tokenised in 6\,s in total, about $1\%$ of one training run per dataset. Tokenisation is linear in the sequence length and in the number of series, parallel over both, and the work at each anchor is set by $\max_k L_k$ and $\max_k m_k$ rather than by the length of the stream, so it also runs online: tokens are emitted every $s$ samples with latency bounded by $s$ and a buffer of $\max_k L_k$ samples, and the optional global branch can be maintained at constant cost per anchor through Chen's identity. The binding constraint is channel width, since the depth-$m$ token has $O(d_x^{m}/m)$ coordinates. 

\begin{table}[h]
\centering
\caption{Compute requirements for the UEA classification tasks of Section~\ref{sec:experiments}. \textbf{Bold} indicates the lowest value per dataset for parameter count ($|\theta|$), GPU memory usage (MB), and time (seconds) per 1000 steps. All statistics were recorded on an NVIDIA RTX 4090.} 
\label{tab:uea_compute}
\scriptsize
\setlength{\tabcolsep}{2.8pt}
\renewcommand{\arraystretch}{1.05}
\resizebox{\textwidth}{!}{%
\begin{tabular}{llccccccccccc}
\toprule
 &  & NRDE & NCDE & Log-NCDE & LRU & S5 & Mamba & S6 & LinOSS-IMEX & LinOSS-IM & D-LinOSS & LogSig-SSM \\
\midrule
\multirow{3}{*}{EigenWorms}
 & $|\theta|$ & 105110 & 166789 & 37977 & 101129 & 22007 & 27381 & \textbf{15045} & 26119 & 134279 & 134279 & 52933 \\
 & mem.      & 2506   & 2484   & 2510  & 10716  & 6646  & 13486 & 7922  & 6556  & 3488   & 3488   & \textbf{362}\\
 & time      & 5386   & 24595  & 1956  & 94     & 31    & 122   & 68    & 37    & 14     & 14     & \textbf{4}\\
\midrule
\multirow{3}{*}{SCP1}
 & $|\theta|$ & 117187 & 166274 & 91557 & 25892  & 226328 & 184194 & \textbf{24898} & 447944 & 991240 & 992776 & 310786\\
 & mem.      & 716    & 694    & 724   & 960    & 1798   & 1110   & 904   & 4768   & 4772   & 4790   & \textbf{202}\\
 & time      & 1014   & 973    & 635   & 9      & 17     & 7      & \textbf{3}     & 42     & 38     & 38     & \textbf{3}\\
\midrule
\multirow{3}{*}{SCP2}
 & $|\theta|$ & 200707 & 182914 & 36379 & 26020  & \textbf{5652}  & 356290 & 26018 & 448072 & 399112 & 399496 & 352642\\
 & mem.      & 712    & 692    & 714   & 954    & 762   & 2460   & 1222  & 4772   & 2724   & 2736   & \textbf{76}\\
 & time      & 1404   & 1251   & 583   & 9      & 9     & 32     & 7     & 55     & 22     & 22     & \textbf{2}\\
\midrule
\multirow{3}{*}{Ethanol}
 & $|\theta|$ & 93212  & 133252 & 31452 & 76522  & 76214 & 1032772 & \textbf{5780}  & 70088  & 6728   & 71112  & 158516\\
 & mem.      & 712    & 692    & 710   & 1988   & 1520  & 4876    & 938   & 4766   & 1182   & 4774   & \textbf{82}\\
 & time      & 2256   & 2217   & 2056  & 16     & 9     & 255     & \textbf{4}    & 48     & 8      & 37     & \textbf{4}\\
\midrule
\multirow{3}{*}{Heartbeat}
 & $|\theta|$ & 15657742 & 1098114 & 168320 & 338820 & 158310 & 1034242 & 6674 & 29444 & 10936 & \textbf{4356} & 396546\\
 & mem.      & 6860     & 1462    & 2774   & 1466   & 1548   & 1650    & 606  & 922   & 928   & 672  & \textbf{54}\\
 & time      & 9539     & 1177    & 826    & 8      & 11     & 34      & 4    & 4     & 7     & 4    & \textbf{2}\\
\midrule
\multirow{3}{*}{Motor}
 & $|\theta|$ & 1134395 & 186962 & 81391 & 107544 & \textbf{17496} & 228226 & 52802 & 106024 & 91844 & 20598 & 70034\\
 & mem.      & 4552    & 4534   & 4566  & 8646   & 4616  & 3120   & 4056  & 12708  & 4510  & 4518  & \textbf{78}\\
 & time      & 7616    & 3778   & 730   & 51     & 16    & 35     & 34    & 128    & 11    & 20    & \textbf{5}\\
\bottomrule
\end{tabular}}
\end{table}

\begin{table}[h]
\centering
\caption{Log-signature tokenisation cost for LogSig-SSM on the UEA datasets, using \texttt{iisignature} \cite{reizenstein2018iisignaturelibraryefficientcalculation} on CPU: wall-clock seconds for the whole dataset (mean\,$\pm$\,std over five seeds) and the same time as a percentage of one training run.}
\label{tab:uea_sigtime}
\small
\begin{tabular}{lrrcc}
\toprule
Dataset & $d_x$ & $T$ & Tokenisation (s) & \% of one training run \\
\midrule
EigenWorms  & 6  & 17,984 & $1.587 \pm 0.004$ & 1.2 \\
SCP1        & 6  & 896    & $0.887 \pm 0.051$ & 0.8 \\
SCP2        & 7  & 1,152  & $0.052 \pm 0.026$ & 0.1 \\
Ethanol     & 3  & 1,751  & $0.113 \pm 0.031$ & 0.1 \\
Heartbeat   & 61 & 405    & $0.959 \pm 0.026$ & 0.9 \\
Motor       & 64 & 3,000  & $2.469 \pm 0.103$ & 1.5 \\
\midrule
All six     &    &        & $6.07$            & \\
\bottomrule
\end{tabular}
\end{table}

\newpage
\newpage
\section{Ablation Studies on UEA}
\label{app:ablations_uea}
Table~\ref{tab:unified_ablation} gives the per-dataset results behind Table~\ref{tab:ablation_ladder}. Every variant uses the same two-stage Optuna protocol and budget as LogSig-SSM (Appendix~\ref{app:hyperparams}), with backbone-appropriate ranges substituted where needed.

\textbf{Tokeniser variants (Mamba backbone fixed).}
\begin{itemize}
\item \textbf{Mamba (no tokenisation):} the raw input sequence is processed directly by Mamba, with no compression.
\item \textbf{ConvDepthwise-Mamba:} a 1-D depthwise (per-channel) convolutional layer over each window, projecting to the same token dimension as LogSig-SSM. This isolates compression with no cross-channel mixing.
\item \textbf{ConvMixing-Mamba:} a standard 1-D convolutional layer over each window, adding linear cross-channel mixing on top of compression.
\item \textbf{Catch24-Mamba}: Catch24 statistical summaries per window, i.e.\ the 22 catch22 features~\cite{lubba_catch22_2019} together with the mean and standard deviation of the window, computed separately for each channel.
\item \textbf{Wst-Mamba:} 1-D wavelet-scattering coefficients (Kymatio \texttt{Scattering1D} with Morlet wavelets) per window, capturing multi-scale spectral content.
\end{itemize}
Wst-Mamba is the only variant ahead of LogSig-SSM on any dataset, on EthanolConcentration ($53.5$ against $51.4$, within one standard deviation).

\textbf{Backbone variants (log-signature tokeniser fixed).} The same log-signature tokens are fed to Mamba (selective recurrence), S5 (LTI recurrence) and a standard Transformer (attention). LogSig-SSM (Mamba) leads on each of the six datasets, with the largest margin on EthanolConcentration, where discriminative information is concentrated in content-defined regions of the sequence.

\begin{table}[H]
\centering
\caption{UEA ablations (\% accuracy $\uparrow$, mean\,$\pm$\,std over five splits). Top: tokeniser variants with the Mamba backbone fixed. Bottom: backbone variants with the log-signature tokeniser fixed. All variants share the same window positions, token dimensions where applicable, and Optuna budget (Appendix~\ref{app:hyperparams}). Bold\,=\,best per column.}
\label{tab:unified_ablation}
\scriptsize
\setlength{\tabcolsep}{3.5pt}
\renewcommand{\arraystretch}{0.9}
\begin{tabular}{lccccccc}
\toprule
Model & EigenWorms & SCP1 & SCP2 & Ethanol & Heartbeat & Motor & Avg \\
\midrule
\multicolumn{8}{l}{\textit{Tokeniser variants (Mamba backbone)}} \\
Mamba (raw input)     & $70.9{\scriptstyle\pm15.8}$ & $80.7{\scriptstyle\pm1.4}$ & $48.2{\scriptstyle\pm3.9}$ & $27.9{\scriptstyle\pm4.5}$ & $76.2{\scriptstyle\pm3.8}$ & $47.7{\scriptstyle\pm4.5}$ & $58.6$ \\
ConvDepthwise-Mamba   & $78.4{\scriptstyle\pm3.0}$  & $83.9{\scriptstyle\pm5.8}$ & $45.2{\scriptstyle\pm3.8}$ & $35.8{\scriptstyle\pm12.5}$ & $68.7{\scriptstyle\pm4.3}$ & $48.6{\scriptstyle\pm2.4}$ & $60.1$ \\
ConvMixing-Mamba      & $85.9{\scriptstyle\pm11.8}$ & $81.4{\scriptstyle\pm3.4}$ & $51.0{\scriptstyle\pm3.1}$ & $33.6{\scriptstyle\pm5.6}$ & $74.9{\scriptstyle\pm10.3}$ & $50.6{\scriptstyle\pm3.8}$ & $62.9$ \\
Catch24-Mamba         & $72.5{\scriptstyle\pm11.0}$ & $82.4{\scriptstyle\pm4.7}$ & $44.5{\scriptstyle\pm4.2}$ & $34.5{\scriptstyle\pm5.7}$ & $66.8{\scriptstyle\pm6.1}$ & $52.8{\scriptstyle\pm7.0}$ & $58.9$ \\
Wst-Mamba             & $89.1{\scriptstyle\pm12.5}$ & $85.9{\scriptstyle\pm1.2}$ & $53.1{\scriptstyle\pm5.6}$ & $\mathbf{53.5{\scriptstyle\pm4.0}}$ & $72.9{\scriptstyle\pm6.3}$ & $51.0{\scriptstyle\pm9.5}$ & $67.6$ \\
\midrule
\multicolumn{8}{l}{\textit{Backbone variants (log-signature tokeniser)}} \\
LogSig-S5             & $85.0{\scriptstyle\pm4.5}$  & $87.5{\scriptstyle\pm2.3}$ & $54.4{\scriptstyle\pm3.8}$ & $25.6{\scriptstyle\pm2.7}$ & $71.6{\scriptstyle\pm8.5}$ & $53.7{\scriptstyle\pm4.7}$ & $63.0$ \\
LogSig-Transformer    & $93.1{\scriptstyle\pm4.9}$  & $83.9{\scriptstyle\pm3.8}$ & $53.7{\scriptstyle\pm3.1}$ & $36.0{\scriptstyle\pm3.4}$ & $72.1{\scriptstyle\pm8.0}$ & $53.7{\scriptstyle\pm4.6}$ & $65.4$ \\
\midrule
\textbf{LogSig-SSM (ours)} & $\mathbf{95.3{\scriptstyle\pm3.6}}$ & $\mathbf{88.4{\scriptstyle\pm1.9}}$ & $\mathbf{58.4{\scriptstyle\pm2.9}}$ & $51.4{\scriptstyle\pm6.9}$ & $\mathbf{76.9{\scriptstyle\pm9.7}}$ & $\mathbf{59.9{\scriptstyle\pm3.4}}$ & $\mathbf{71.7}$ \\
\bottomrule
\end{tabular}
\end{table}

\subsection{Signature versus log-signature tokens}
\label{app:sig_vs_logsig}
The truncated signature and log-signature are related by a polynomial bijection, so at equal depth they carry identical information and any continuous functional of one is a continuous functional of the other. What the log-signature gives up is universality for linear readouts, which the nonlinear readout of LogSig-SSM does not need. We use the log-signature because it is the driver with which the models we build on are defined (the log-ODE method underlying NRDEs \cite{morrill_neural_2021} and Log-NCDEs \cite{walker_log_2024}), because it is the redundancy-free coordinate system for that content, and because \texttt{iisignature} accumulates it directly in a Lie basis rather than taking a logarithm of the signature. Table~\ref{tab:sig_vs_logsig} compares the two token types directly in accuracy and tokenisation time on EigenWorms, SCP1 and SCP2, the three datasets on which the second-order terms are load-bearing (Appendix~\ref{app:depth_study}): the log-signature coordinates of the reported configuration are replaced by signature coordinates of the same depth with every other hyperparameter fixed, and since the two coincide at depth one the comparison isolates how the higher levels are coordinatised. The log-signature leads on all three, by $3.1$ points on average, and is about twice as fast to compute.
\begin{table}[H]
\centering
\small
\caption{Signature against log-signature tokens at the same depth and configuration (Table~\ref{tab:hparams_sig_mamba_sig}): test accuracy (\%, mean over five seeds) and whole-dataset tokenisation time with \texttt{iisignature} (s, best of three).}
\label{tab:sig_vs_logsig}
\setlength{\tabcolsep}{6pt}
\begin{tabular}{lcccc}
\toprule
 & \multicolumn{2}{c}{Accuracy (\%)} & \multicolumn{2}{c}{Tokenisation (s)} \\
\cmidrule(lr){2-3}\cmidrule(lr){4-5}
Dataset & Log-signature & Signature & Log-signature & Signature \\
\midrule
EigenWorms & $\mathbf{95.3}$ & $94.4$ & $1.56$ & $3.73$ \\
SCP1       & $\mathbf{88.4}$ & $85.7$ & $0.85$ & $1.34$ \\
SCP2       & $\mathbf{58.4}$ & $52.9$ & $0.04$ & $0.05$ \\
\midrule
Average / total & $\mathbf{80.7}$ & $77.6$ & $2.45$ & $5.12$ \\
\bottomrule
\end{tabular}
\end{table}

\section{Truncation Depth: Matched-Budget Comparison}
\label{app:depth_study}
A level-one log-signature over a window is its increment and carries no cross-channel content, second-order terms are where cross-channel interactions (L\'evy areas) enter the token. To decide the depth without per-dataset selection, we ran, for each dataset, a fresh 100-trial Optuna search over the tokeniser hyperparameters (window lengths, number of scales, stride) with the backbone fixed at the reported configuration (Table~\ref{tab:hparams_sig_mamba_sig}) and the depth forced to one or to two on every window, so that both arms receive a matched budget at their own best tokeniser configuration. Table~\ref{tab:depth_study} reports test accuracy over five seeds.

\begin{table}[H]
\centering
\caption{Matched-budget comparison of truncation depth $m_k=1$ against $m_k=2$ on every window (UEA accuracy \%; mean\,$\pm$\,std over five seeds). $\Delta$ is the gain of depth two. The depth-two column is the configuration reported in Table~\ref{tab:uea_main}.}
\label{tab:depth_study}
\small
\begin{tabular}{lcccc}
\toprule
Dataset & $d_x$ & $m_k=1$ & $m_k=2$ & $\Delta$ \\
\midrule
EigenWorms & 6  & $88.1{\scriptstyle\pm6.4}$ & $95.3{\scriptstyle\pm3.6}$ & $+7.2$ \\
SCP1      & 6  & $84.1{\scriptstyle\pm1.4}$ & $88.4{\scriptstyle\pm1.9}$ & $+4.3$ \\
SCP2      & 7  & $55.6{\scriptstyle\pm4.9}$ & $58.4{\scriptstyle\pm2.9}$ & $+2.8$ \\
Ethanol   & 3  & $51.0{\scriptstyle\pm4.2}$ & $51.4{\scriptstyle\pm6.9}$ & $+0.4$ \\
Heartbeat & 61 & $78.3{\scriptstyle\pm9.4}$ & $76.9{\scriptstyle\pm9.7}$ & $-1.4$ \\
Motor     & 64 & $59.6{\scriptstyle\pm7.5}$ & $59.9{\scriptstyle\pm3.4}$ & $+0.3$ \\
Average (UEA) & & $69.5$ & $71.7$ & $+2.3$ \\
\bottomrule
\end{tabular}
\end{table}

The two directions do not cost the same. Forcing depth one costs $2.3$ points on average, with losses of $7.2$, $4.3$ and $2.8$ on the three datasets with few channels, where the second-order terms are load-bearing: on these three datasets the depth-one arm is level with the strongest alternative tokeniser (Wst-Mamba, Table~\ref{tab:unified_ablation}; $-1.0$, $-1.8$, $+2.5$), while depth two leads it by $6.2$, $2.5$ and $5.3$ points. Forcing depth two costs at most $1.4$ points on any dataset and about $0.2$ points on average relative to selecting the better depth per dataset, within seed noise everywhere. On Heartbeat and MotorImagery the depth-two token per branch ($1{,}891$ and $2{,}080$ coordinates) is $6.6\times$ and $7.9\times$ wider than the training set, so the second-order coordinates are under-determined and sit on a plateau that we expect to relax with more data. On Ethanol the discriminative structure is carried by the increments, and the extra depth neither helps nor hurts. Depth two is therefore a safe default, its cost being tokenisation time on wide inputs (Table~\ref{tab:uea_sigtime}), and we report it throughout.

\section{Signal Visualisation}
\label{app:signals}
We visualise representative signals from the six UEA classification tasks used in our evaluation (Figure~\ref{fig:uea_signals}). The goal is to highlight the qualitative differences in temporal structure across datasets, which helps explain the heterogeneous empirical performance observed in Table~\ref{tab:uea_main}.

The datasets separate naturally into two regimes. \textbf{Five datasets} (SelfRegulationSCP1, SelfRegulationSCP2, MotorImagery, Heartbeat, EigenWorms) exhibit \emph{high-frequency, oscillatory, or noise-dominated dynamics}, where discriminative information is often localised in short-time patterns or frequency content. In contrast, \textbf{EthanolConcentration} consists of \emph{smooth absorption spectra} with a single broad peak, where class differences manifest as \emph{global shifts in shape and location} over long time horizons.

This distinction is important for interpreting our results. Multi-scale log-signature tokens summarise the \emph{geometry of the path} over extended intervals, capturing global shape characteristics and cross-channel relationships. This is well matched to signals such as EthanolConcentration, where discriminative structure is inherently low-frequency and global. This supports the interpretation that the gains of LogSig-SSM arise from its ability to capture \emph{multi-scale path geometry}, rather than compression alone.

\begin{figure}[h]
    \centering
    \includegraphics[width=\linewidth]{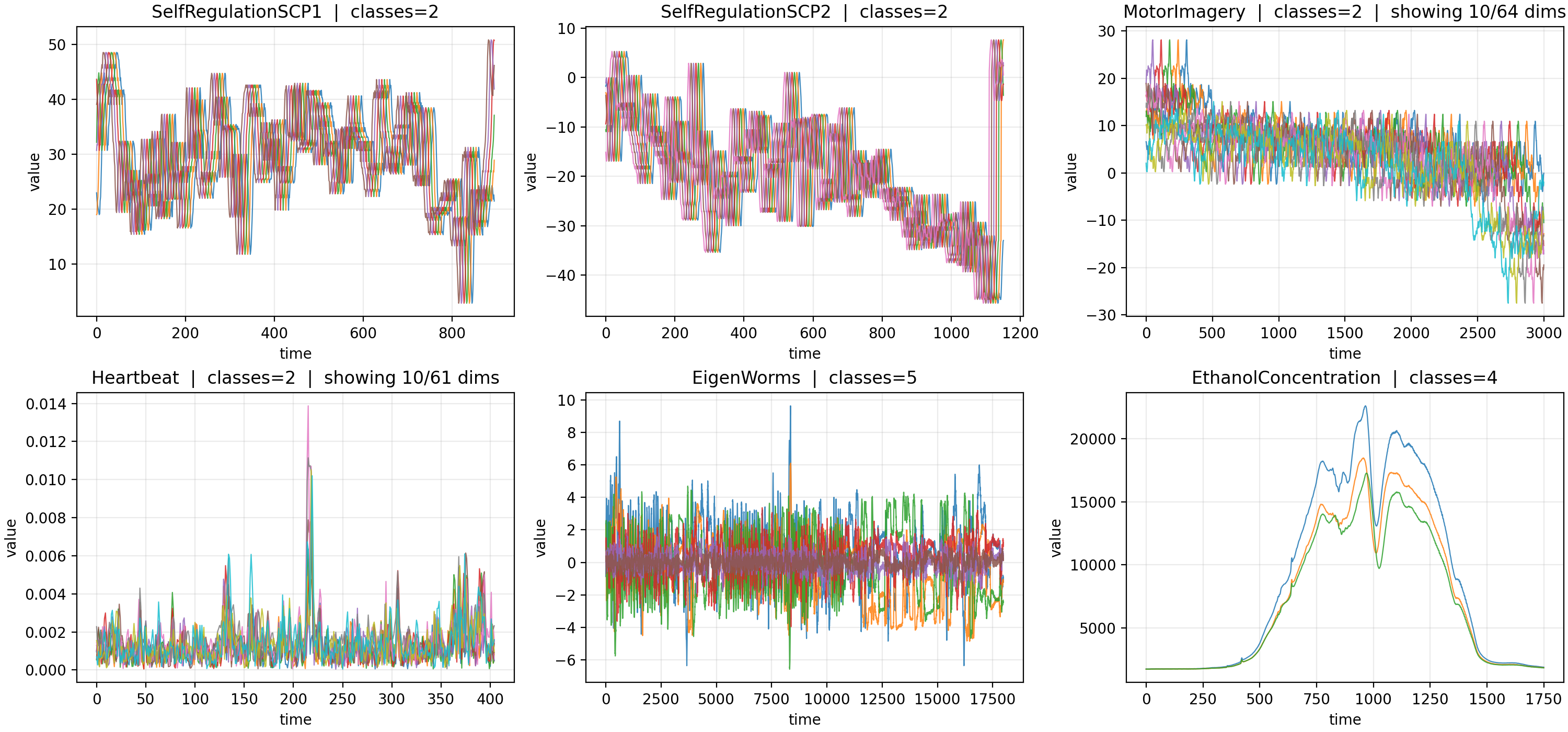}
    \caption{
    Representative time-series from the six UEA datasets. EthanolConcentration exhibits smooth, low-frequency structure with global shape variation, while the other datasets are dominated by high-frequency or noisy dynamics.
    }
    \label{fig:uea_signals}
\end{figure}

\section{Selected Hyperparameters}
\label{app:hyperparams}
\subsection{Hyperparameter optimisation protocol}
\label{app:hyperparams_procedure}

We optimise log-signature tokenisation and the Mamba backbone using a two-stage Optuna~\cite{akiba_optuna_2019} protocol with $100$ trials per stage. Stage 1 searches the tokenisation space with the backbone fixed at a small reference Mamba ($d{=}64$, $2$ layers, $N{=}32$, convolution width $3$, expansion $2$). Stage 2 searches the Mamba backbone jointly with a local neighbourhood of the stage-1 tokenisation. The same protocol and budget are used for the LogSig-S5 and LogSig-Transformer ablations, with backbone-appropriate hyperparameter ranges substituted in Stage 2.

\paragraph{Stage 1 tokenisation ranges.}
Stride $s \in \{2, \dots, 199\}$; number of multi-scale branches $K \in \{1, \dots, 5\}$, bounded by the number of distinct admissible window lengths; window lengths $L_k$ drawn from a $T$-dependent geometric menu with values $T/r$ for $r \in \{4, \dots, 64\}$, each lower-bounded by a small constant; an optional global branch (window equal to the full sequence). On UEA the truncation depth is fixed at $m_k=2$ on every window. On PPG-DaLiA, Weather and PhysioNet the per-branch depth was searched in $\{1, \dots, m_{\max}\}$ with $m_{\max}{=}4$ if $d_x \leq 20$ and $m_{\max}{=}2$ otherwise, and every selected configuration contains a depth-two branch.

\paragraph{Stage 2 local tokenisation neighbourhood.}
Stride is scaled by $\{0.5, 0.75, 1.0, 1.5, 2.0\}$ around the stage-1 best; the number of branches $K^* \pm 1$; per-branch window lengths are scaled by $\{0.5, 0.75, 1.0, 1.5, 2.0\}$; where the depth is searched, depths within $m_k^* \pm 1$, clamped to $[1, m_{\max}]$.

\paragraph{Stage 2 backbone ranges (Mamba).}
Learning rate $\in \{10^{-3}, 10^{-4}, 10^{-5}\}$; model dimension $d \in \{16, 64, 128\}$; blocks $\in \{2, 4, 6\}$; state size $N \in \{16, 64, 256\}$; convolution width $\in \{2, 3, 4\}$; expansion $\in \{1, 2, 4\}$; per-feature normalisation $\in \{\text{True}, \text{False}\}$. For LogSig-S5, the layer/block/state-size ranges are searched in place of the Mamba-specific hyperparameters; for LogSig-Transformer, number of layers, attention heads, feedforward dimension, dropout, and learning rate are searched.

\paragraph{Comparison to prior protocols.}
For fairness with prior work, we adopt the same Mamba hyperparameter grid as the Log-NCDE protocol~\cite{walker_log_2024} and its models following this protocol~\cite{rusch_oscillatory_2025,boyer_learning_2025}: learning rate, time-channel inclusion, hidden dimension, number of layers, state dimension, convolution width, and expansion factor, with the same value ranges on each axis. \citet{walker_log_2024} search this grid exhaustively, evaluating $3^6 \cdot 2 = 1458$ Mamba configurations (and a comparable $1296$-point grid for Log-NCDE itself). We instead use Optuna to evaluate $100$ Mamba configurations in Stage~2 over the same axes and ranges. Two consequences follow. First, our Mamba backbone is searched over the same space as the strongest published Mamba results on these benchmarks, so the comparison in Table~\ref{tab:uea_main} is structurally fair. Second, our budget is smaller in raw trial count but is allocated adaptively rather than uniformly. Stage~1 separately searches the tokenisation space, so the total budget per dataset is $200$ trials, of which $100$ involve backbone selection.

\paragraph{Search-free rule.} The search is not required to beat raw SSMs. With the Mamba backbone fixed at the reported LogSig-SSM configuration of each dataset and every tokeniser hyperparameter set from the sequence length alone ($s=T/100$, lengths $[T/25,\,T/100]$, depths $[2,1]$), applied unchanged to all six datasets, LogSig-SSM reaches $64.0\%$ average UEA accuracy (Table~\ref{tab:rule_sweep}), above S5 ($61.8$), LRU ($61.7$) and Mamba ($58.6$), and searching the tokeniser adds the remaining $7.7$ points. The full per-axis study is provided with the code.

\begin{table}[H]
\centering
\small
\caption{Search-free rule on UEA (accuracy \%, mean\,$\pm$\,std over five seeds): tokeniser hyperparameters set from the sequence length alone, Mamba backbone as in Table~\ref{tab:hparams_sig_mamba_sig}, against the searched configuration and the raw SSM baselines of Table~\ref{tab:uea_main}.}
\label{tab:rule_sweep}
\setlength{\tabcolsep}{4pt}
\begin{tabular}{lccccccc}
\toprule
Model & EigenWorms & SCP1 & SCP2 & Ethanol & Heartbeat & Motor & Avg \\
\midrule
LogSig-SSM (searched) & $95.3{\scriptstyle\pm3.6}$ & $88.4{\scriptstyle\pm1.9}$ & $58.4{\scriptstyle\pm2.9}$ & $51.4{\scriptstyle\pm6.9}$ & $76.9{\scriptstyle\pm9.7}$ & $59.9{\scriptstyle\pm3.4}$ & $71.7$ \\
LogSig-SSM (rule, no search) & $92.5{\scriptstyle\pm6.8}$ & $85.5{\scriptstyle\pm3.5}$ & $51.9{\scriptstyle\pm8.9}$ & $32.3{\scriptstyle\pm2.1}$ & $72.2{\scriptstyle\pm7.3}$ & $49.7{\scriptstyle\pm4.8}$ & $64.0$ \\
S5    & $81.1{\scriptstyle\pm3.7}$ & $89.9{\scriptstyle\pm4.6}$ & $50.5{\scriptstyle\pm2.6}$ & $24.1{\scriptstyle\pm4.3}$ & $77.7{\scriptstyle\pm5.5}$ & $47.7{\scriptstyle\pm5.5}$ & $61.8$ \\
LRU   & $87.8{\scriptstyle\pm2.8}$ & $82.6{\scriptstyle\pm3.4}$ & $51.2{\scriptstyle\pm3.6}$ & $21.5{\scriptstyle\pm2.1}$ & $78.4{\scriptstyle\pm6.7}$ & $48.4{\scriptstyle\pm5.0}$ & $61.7$ \\
Mamba & $70.9{\scriptstyle\pm15.8}$ & $80.7{\scriptstyle\pm1.4}$ & $48.2{\scriptstyle\pm3.9}$ & $27.9{\scriptstyle\pm4.5}$ & $76.2{\scriptstyle\pm3.8}$ & $47.7{\scriptstyle\pm4.5}$ & $58.6$ \\
\bottomrule
\end{tabular}
\end{table}

\subsection{Log-signature tokeniser}
\label{app:hyperparams_signature}

\begin{table}[H]
\centering
\footnotesize
\caption{Selected hyperparameters for \textbf{LogSig-SSM}. LR: learning rate; Hidden: model dimension $d$; State: state size $N$; Blocks: number of Mamba blocks; Time: time channel appended to the input; Conv: convolution width; Exp: expansion factor; Norm: per-feature normalisation; Lengths/Depths: window lengths $L_k$ and truncation depths $m_k$; Global: depth of the optional global branch (-- = none).}
\label{tab:hparams_sig_mamba_sig}
\setlength{\tabcolsep}{3pt}
\resizebox{\textwidth}{!}{%
\begin{tabular}{lcccccccclcc}
\toprule
Dataset & LR & Hidden & State & Blocks & Time & Conv & Exp & Norm & Lengths / Depths & Stride & Global \\
\midrule
EigenWorms        & 1e-3 & 64  & 64  & 2 & F & 2 & 1 & T & [281] / [2]              & 21 & -- \\
Ethanol           & 1e-3 & 16  & 256 & 6 & T & 2 & 2 & T & [18, 132] / [2, 2]       & 66 & -- \\
Heartbeat         & 1e-3 & 128 & 64  & 2 & F & 4 & 1 & T & [143] / [2]              & 26 & -- \\
Motor             & 1e-4 & 16  & 16  & 2 & T & 3 & 1 & F & [51, 16] / [2, 2]        & 50 & -- \\
SCP1              & 1e-3 & 128 & 64  & 2 & F & 3 & 2 & T & [56] / [2]               & 5  & 1 \\
SCP2              & 1e-4 & 128 & 16  & 2 & T & 2 & 3 & T & [36] / [2]               & 46 & -- \\
PPG-DaLiA         & 1e-3 & 16  & 128 & 4 & T & 3 & 4 & T & [64, 8, 8] / [2, 2, 1]   & 64 & -- \\
Weather           & 1e-3 & 32  & 32  & 6 & F & 2 & 2 & T & [20, 30, 90] / [1, 2, 1] & 7  & 1 \\
PhysioNet (OI)    & 1e-3 & 64  & 128 & 6 & F & 4 & 2 & T & [6, 2, 2] / [1, 2, 1]    & 1  & -- \\
PhysioNet (no OI) & 5e-4 & 32  & 16  & 4 & F & 4 & 2 & T & [4, 6, 22] / [2, 2, 2]   & 1  & -- \\
\bottomrule
\end{tabular}}
\end{table}

\subsection{Ablation variants}
\label{app:hyperparams_variants}
Tables~\ref{tab:hparams_sig_s5_sig}--\ref{tab:hparams_wst_mamba} list the selected configurations of the LogSig-S5, LogSig-Transformer, Catch24-Mamba and Wst-Mamba variants.

\begin{table}[H]
\centering
\footnotesize
\caption{Selected hyperparameters for \textbf{LogSig-S5} (log-signature tokens + S5). Blocks are reported as layers$\times$blocks.}
\label{tab:hparams_sig_s5_sig}
\setlength{\tabcolsep}{4pt}
\begin{tabular}{lccccccccccc}
\toprule
Dataset & LR & Hidden & State & Blocks & Time & Norm & Lengths & Depths & Stride & Global \\
\midrule
EigenWorms & 1e-3 & 64 & 16  & 4$\times$4 & F & F & [562, 281]      & [2, 2]    & 57 & F \\
Ethanol    & 1e-3 & 16 & 64  & 2$\times$2 & F & T & [109, 43]       & [2, 1]    & 70 & F \\
Heartbeat  & 1e-5 & 64 & 16  & 6$\times$2 & T & T & [14]            & [1]       & 3  & F \\
Motor      & 1e-4 & 64 & 16  & 4$\times$4 & T & F & [93, 250, 125]  & [2,1,2]   & 35 & T \\
SCP1       & 1e-3 & 16 & 256 & 6$\times$4 & F & F & [112, 18]       & [1, 1]    & 30 & F \\
SCP2       & 1e-4 & 16 & 64  & 4$\times$2 & F & T & [18, 144, 192]  & [4,2,3]   & 13 & T \\
\bottomrule
\end{tabular}
\end{table}

\begin{table}[H]
\centering
\footnotesize
\caption{Selected hyperparameters for \textbf{LogSig-Transformer} (log-signature tokens + Transformer). FFN: feed-forward dimension; Global: depth of the optional global branch (-- = none).}
\label{tab:hparams_sig_transformer}
\setlength{\tabcolsep}{2.5pt}
\begin{tabular}{lcccccccccccc}
\toprule
Dataset & LR & Hidden & Heads & FFN & State & Blocks & Time & Norm & Lengths & Depths & Stride & Global \\
\midrule
EigenWorms & 1e-3 & 128 & 4 & 256 & 64 & 2 & T & T & [374, 899]    & [2, 1]    & 55 & -- \\
Ethanol    & 1e-3 & 16  & 8 & 128 & 64 & 6 & T & F & [22, 48, 114] & [1, 1, 1] & 74 & -- \\
Heartbeat  & 1e-5 & 16  & 8 & 256 & 64 & 4 & T & T & [16, 25]      & [1, 1]    & 14 & -- \\
Motor      & 1e-4 & 16  & 8 & 512 & 64 & 4 & T & T & [250, 46]     & [1, 1]    & 46 & -- \\
SCP1       & 1e-4 & 16  & 8 & 512 & 64 & 4 & T & T & [74]          & [1]       & 14 & 1 \\
SCP2       & 1e-3 & 128 & 8 & 512 & 64 & 2 & T & T & [68]          & [2]       & 76 & -- \\
\bottomrule
\end{tabular}
\end{table}

\begin{table}[H]
\centering
\footnotesize
\caption{Selected hyperparameters for \textbf{Catch24-Mamba} (Catch24 tokens).}
\label{tab:hparams_sig_mamba_catch}
\setlength{\tabcolsep}{3pt}
\begin{tabular}{lcccccccccc}
\toprule
Dataset & LR & Hidden & State & Blocks & Time & Conv & Exp & Norm & Lengths & Stride \\
\midrule
EigenWorms & 1e-3 & 64  & 256 & 6 & T & 2 & 4 & T & [70, 749, 499, 70, 35] & 135 \\
Ethanol    & 1e-3 & 128 & 256 & 4 & T & 3 & 3 & T & [109, 350, 13, 109, 350] & 79  \\
Heartbeat  & 1e-5 & 128 & 64  & 2 & T & 3 & 2 & F & [11, 16, 50, 8, 3] & 38 \\
Motor      & 1e-3 & 64  & 16  & 6 & F & 3 & 3 & T & [187, 11] & 163 \\
SCP1       & 1e-4 & 64  & 16  & 2 & T & 4 & 2 & T & [179, 9, 179, 32, 224] & 98 \\
SCP2       & 1e-3 & 64  & 256 & 6 & T & 4 & 3 & T & [192, 48, 115, 192] & 67\\
\bottomrule
\end{tabular}
\end{table}

\begin{table}[H]
\centering
\footnotesize
\caption{Selected hyperparameters for \textbf{Wst-Mamba} (wavelet-scattering tokens, Kymatio \texttt{Scattering1D} with Morlet wavelets).}
\label{tab:hparams_wst_mamba}
\setlength{\tabcolsep}{4pt}
\begin{tabular}{lcccccccc}
\toprule
Dataset & LR & Hidden & State & Blocks & Conv & Exp & WST Lengths & Stride \\
\midrule
EigenWorms & 1e-3 & 128 & 64  & 6 & 2 & 1 & [749, 2997]           & 48  \\
Ethanol    & 1e-4 & 128 & 64  & 2 & 3 & 4 & [18, 291, 62, 48]     & 66  \\
Heartbeat  & 1e-4 & 64  & 64  & 2 & 3 & 4 & [67, 67, 40]          & 75  \\
Motor      & 1e-5 & 16  & 16  & 2 & 2 & 1 & [375, 125, 150, 187]  & 125 \\
SCP1       & 1e-5 & 128 & 256 & 2 & 4 & 2 & [179, 28]             & 122 \\
SCP2       & 1e-3 & 16  & 16  & 2 & 3 & 1 & [48, 4, 96, 48, 192]  & 141 \\
\bottomrule
\end{tabular}
\end{table}

\newpage

\end{document}